%% file: 00_main.tex
\documentclass{article}

\usepackage[T1]{fontenc}
\usepackage[utf8]{inputenc}
\usepackage{lmodern}

\makeatletter
\def\blfootnote{\gdef\@thefnmark{}\@footnotetext}
\makeatother

\usepackage[bottom]{footmisc}

\usepackage{cancel}
\usepackage{hyperref}
\usepackage{url}
\usepackage{fancyhdr}
\usepackage{microtype}
\usepackage{subcaption}
\usepackage{dirtytalk}
\usepackage{booktabs}
\usepackage{enumitem}
\usepackage{float}
\usepackage{multirow}
\usepackage{multicol}
\usepackage{amsmath}
\usepackage{amssymb}
\usepackage{epigraph}
\usepackage{graphicx}
\usepackage{amsthm}
\usepackage{amsfonts}
\usepackage{mathtools}
\usepackage{fullpage}
\usepackage[nice]{nicefrac}
\usepackage{algorithm}
\usepackage{algorithmic}
\usepackage{xcolor}
\usepackage[numbers]{natbib}
\usepackage{comment}

\usepackage[toc,page,header]{appendix}


\hypersetup{colorlinks,linkcolor=blue}

\usepackage[capitalize,noabbrev,nameinlink]{cleveref}

\definecolor{amaranth}{rgb}{0.9, 0.17, 0.31}
\definecolor{green}{HTML}{549D54}


\makeatletter
\renewcommand\section{\@startsection{section}{1}{\z@}%
                                    {-2.5ex \@plus -1ex \@minus -.2ex} %
                                    {1.5ex \@plus.2ex} %
                                    {\normalfont\Large\bfseries}}
\makeatother

\makeatletter
\renewcommand\subsection{\@startsection{subsection}{2}{\z@}%
                                    {-1.5ex \@plus -1ex \@minus -.2ex} %
                                    {1ex \@plus .2ex} %
                                    {\normalfont\large\bfseries}}
\makeatother

\title{Certification from Examples is Hard for Circuits and Transformers under Minimal Overparametrization}

\date{}

\theoremstyle{plain}

\newtheorem{theorem}{Theorem}[section]
\newtheorem{proposition}{Proposition}[section]

\newtheorem{lemma}{Lemma}[section]

\newtheorem{corollary}{Corollary}[section]
\theoremstyle{definition}
\newtheorem{definition}{Definition}[section]

\theoremstyle{plain}
\newtheorem{remark}{Remark}[section]
\newenvironment{proofsketch}
  {\begin{proof}[Proof sketch]}
  {\end{proof}}

\newcommand{\TCzero}{\mathrm{TC}^0}
\newcommand{\ACzero}{\mathrm{AC}^0}
\newcommand{\NCone}{\mathrm{NC}^1}
\newcommand{\VS}{\mathrm{VS}}
\newcommand{\cert}{\operatorname{cert}}
\newcommand{\err}{\operatorname{err}}
\newcommand{\abs}{\mathrm{abs}}
\newcommand{\rel}{\mathrm{rel}}
\newcommand{\errabs}{\mathop{\mathrm{err}}\nolimits_{\mathrm{abs}}}
\newcommand{\errrel}{\mathop{\mathrm{err}}\nolimits_{\mathrm{rel}}}
\newcommand{\poly}{\mathrm{poly}}
\newcommand{\N}{\mathbb{N}}

\newcommand{\certrel}{%
  \mathop{\cert}\nolimits^{\mathrm{rel}}%
}
\newcommand{\certabs}{%
  \mathop{\cert}\nolimits^{\mathrm{abs}}%
}

\author{
    Artur Back de Luca\qquad
    Kimon Fountoulakis\\
    \vspace{-1mm}\\
    \normalsize{University of Waterloo}\\
    \vspace{-3mm}\\
    \normalsize{\texttt{\{\href{mailto:abackdel@uwaterloo.ca}{abackdel}, \href{mailto:kimon.fountoulakis@uwaterloo.ca}{kimon.fountoulakis}\}@uwaterloo.ca}}\\
}

\begin{document}
\maketitle

\def\thefootnote{*}
\def\thefootnote{\arabic{footnote}}

\vspace{-5mm}

\begin{abstract}
As state-of-the-art neural networks are deployed on reasoning and algorithmic tasks, exactness guarantees become increasingly important. However, high average-case accuracy can still mask incorrect behaviors. This motivates exact certification, which asks for the smallest set of labeled examples needed to certify that a learned hypothesis equals the target. We show that while some hypotheses are easy to certify, even minimal overparametrization can make certification exponentially hard across several hypothesis classes. For threshold circuits of depth $\ge 2$, adding a single extra gate can force certificate sizes exponential in the input dimension. We show an analogous hardness result for log-precision Transformers with only constant architectural overhead.
We also characterize approximate certification, showing that allowing only polynomially many mistakes still requires exponentially large certificates, whereas constant relative-error guarantees can hide exponentially many failures.
Empirically, we study binary-addition recognition using threshold circuits and Transformers trained with strong supervision. We find that incorrect models can survive extensive checks, with errors that often occur on different inputs and are difficult to detect through uniform sampling.
\end{abstract}

\input{01_intro}
\input{02_review}

\input{03_preliminaries}
\input{04_1_hardness}
\input{04_2_circuits}
\input{04_3_transformers}
\input{04_4_approximate}
\input{05_experiments}
\input{06_conclusion}

\section*{Acknowledgments}
K.~Fountoulakis would like to acknowledge the support of the Natural Sciences and Engineering Research Council of Canada (NSERC). Cette recherche a \'et\'e financ\'ee par le Conseil de recherches en sciences naturelles et en g\'enie du Canada (CRSNG), [RGPIN-2019-04067, DGECR-2019-00147].

A. Back de Luca would like to acknowledge the support of the Natural Sciences and Engineering Research Council of
Canada (NSERC).

This work was supported by RBC Borealis through the RBC Borealis Fellowship Award.

\bibliography{references}
\bibliographystyle{plainnat}
\newpage
\appendix
\input{990_basic}
\input{991_circuits}
\input{992_transformer}
\input{993_experiments}

\end{document}

%% file: 01_intro.tex
\section{Introduction}
\label{sec:intro}

Frontier models can tackle a wide range of tasks, yet studies have revealed inconsistencies in reasoning and algorithmic settings \citep{nezhurina2024alice,malek2025frontier}.
Even when a model achieves high average-case accuracy, such inconsistencies suggest that it may fail to implement the intended reasoning behavior, and these failures are difficult to detect from average-case evaluation alone.

This motivates a broader exact certification question: after a learner returns a candidate hypothesis, how many labeled input-output examples are needed to certify, from behavior alone, that the candidate equals the intended target? \citet{gyorgy2025beyond} flag this certification problem as a foundational challenge for the exact-learning setting, illustrating the worst-case certification in a bit comparison task. In this work, we give a quantitative form of certification hardness, mainly focusing on two formal hypothesis classes closely related to reasoning and computation.

Because certification uses only labeled input-output data, we formalize the question using a perspective closely related to the teaching set literature \citep{anthony1992exact,shinohara1991teachability, goldman1995complexity}. The key quantity is the certificate size, namely the minimum number of labeled examples needed to uniquely identify a target within a hypothesis class. This quantity depends not only on the target but also on the surrounding class. Our main observation is that even a minimal enlargement of the hypothesis class can make certification exponentially hard for every target in the original class, including targets that previously admitted small certificates. For models trained to perform reasoning tasks consistently, this means that even extensive input-output testing may be insufficient to certify correctness.

\paragraph{Contributions.}
Our contributions are as follows.

\begin{enumerate}
    \item We provide a quantitative form of hardness for exact certification by examples. We show that certification can be highly sensitive to the surrounding hypothesis class. Even a constant enlargement can make certification exponentially hard for every target in the original class, despite the presence of targets with small certificates in that class. 
    We instantiate this phenomenon in two settings related to algorithmic computation.
    \begin{enumerate}
        \item In the circuit setting, we study unbounded fan-in classes such as $\TCzero$ and $\ACzero$, together with the bounded fan-in class $\NCone$. 
        Our sharpest result is for $\TCzero$, where circuits of depth $d \ge 2$ with a single additional gate already force certification to require exponentially many examples in the input dimension.

        \item In the Transformer setting, motivated by connections between Transformers and circuit complexity \citep{hao2022formal,merrill2022saturated,merrill2023logic,merrill2024expressive,chiang2025transformers}, we study log-precision Transformers with averaging hard attention (AHAT) \citep{hao2022formal}. We show that a slight overparametrization with one extra attention head and six auxiliary residual coordinates already makes certification hard.
        \item We extend the hardness analysis to approximate certification. We show that allowing only polynomially many absolute mistakes still requires exponentially large certificates, while constant relative-error guarantees may leave exponentially many absolute mistakes.
    \end{enumerate}

    \item We complement the theory with two empirical certification analyses for binary-addition recognition.
    First, we construct a $\TCzero$ circuit for recognizing addition, together with a collection of incorrect circuits, each agreeing with the target except on a different block of inputs.
    We then compute how many of these circuits remain consistent with uniformly sampled labeled examples.
    Second, we study Transformers trained with strong intermediate supervision that pass extensive training and validation checks without errors. On a large held-out test set, we find that some models still make errors, often on different inputs, and that detecting these errors through uniform sampling can require a large number of examples.
\end{enumerate}

%% file: 02_review.tex
\section{Related work}
\label{sec:related_work}

Our certification setting is most closely related to the classical teaching complexity
literature \citep{anthony1992exact, shinohara1991teachability,goldman1995complexity}. Subsequent work develops a range of teacher-learner protocol variants \citep{zilles2011models,doliwa2014recursive,gao2017preference} and relates teaching parameters to structural properties of the hypothesis
class, such as class complexity \citep{doliwa2014recursive,chen2016recursive,hu2017quadratic}. In the circuit setting, the closest work is \citet{servedio2001limits}, who proves worst-case exponential lower bounds on the size of certificates for polynomial-size monotone circuits using input-output examples alone. The main distinction is that our lower bounds do not isolate a single worst-case target. This distinction is sharpest in our threshold-circuit result in \Cref{thm:unbounded}, which shows that for every target in a smaller threshold-circuit class, adding one same-depth gate already makes exact certification exponentially hard.

Related to certification by examples is exact learning, which often studies
active, query-based interaction between a learner and an oracle, including
membership queries, equivalence queries, and related oracle access \citep{angluin1987learning,angluin1988queries,angluin1990negative,hegedus1995generalized,hellerstein1996queries,chase2020bounds}. A central theme in this literature is that learnability depends both on the interaction model and on the hypothesis class the learner may output. For example, proper and improper learning of DNFs exhibit different learnability behavior \citep{pillaipakkamnatt1996limits,hellerstein2005exactdnf}, while richer oracle access can make circuit learning tractable \citep{bshouty1996oracles}. In our input-output-only certification setting, we isolate the role of the hypothesis class and show that even a small enlargement of the class can make certification hard.

In the context of neural networks, several recent works advocate for exact guarantees, rather than average-case accuracy alone, as a foundation for reliable reasoning and algorithmic behavior \citep{gyorgy2025beyond,de2025learning,papazov2025exact}. In particular, \citet{gyorgy2025beyond} emphasize both the importance and the challenge of certifying exact correctness. Our work builds on this perspective by quantifying the size of certificates needed for log-precision Transformers. Most theoretical work on log-precision Transformers has focused on expressive power, whose computation is captured by threshold-circuit models \citep{merrill2022saturated,merrill2023logic,merrill2024expressive,merrill2023parallelism,chiang2025transformers}. While these results characterize what such models can compute, our results address the complementary question of how many input-output examples are needed to certify their behavior. We show that certifying log-precision Transformers within a slightly overparametrized class requires exponentially many examples in the input size, even for exact and near-exact certification.

%% file: 03_preliminaries.tex
\section{Preliminaries and notation}
\label{sec:prelim}
In this section, we introduce the notation and background definitions used throughout the paper.
We write $[n]$ to indicate the set $\{1,2,\dots,n\}$. For a vector $x \in \{0,1\}^n$ and an index set $I \subseteq [n]$, we write $x_I$ for the restriction of $x$ to the coordinates indexed by $I$. For a finite set $S$, we let $\{0,1\}^S = \{\pi:S\to\{0,1\}\}$ denote the set of binary patterns on $S$.

\textbf{Circuit complexity and circuit classes.}
We use circuits as finite semantic hypothesis classes because they let us measure how much computational power is added by one or a few gates. Circuits also connect directly to the Transformer result, since log-precision Transformer expressivity is commonly compared with threshold-circuit computation \citep{merrill2022saturated, merrill2024expressive}.

A Boolean circuit on inputs $x_1,\ldots,x_n$ is a directed acyclic graph with
input nodes, internal gates, and one output gate. Its size $|C|$ is the number
of non-input gates, and its depth is the length of the longest input-to-output
path. At polynomial size, the standard containments are
\[
\ACzero \subsetneq \TCzero \subseteq \NCone .
\]
Here $\ACzero$ consists of constant-depth circuits with unbounded-fan-in
AND/OR gates over input literals, $\TCzero$ adds constant-depth threshold gates,
and $\NCone$ allows bounded-fan-in circuits of logarithmic depth. The first
containment is strict, while strictness of the second is open.

For finite gate budgets, $\ACzero_{d,s}$ and $\TCzero_{d,s}$ denote the $n$-input functions computable by circuits of depth at most $d$ and size at most $s$ of the corresponding type. In $\TCzero_{d,s}$, a threshold gate is defined as
$\mathrm{THR}_{w,\theta}(z)=\mathbf 1[\sum_i w_i z_i\ge \theta]$, with
$w_i,\theta\in\mathbb Z$. For $\TCzero$, the class can also be expressed using Boolean disjunctions/conjunctions and majority gates. All of these gates can be expressed as a threshold gate, but the converse simulation can change the finite gate budget, so
the stated $\TCzero$ bounds should be read under the threshold gate convention. The class
$\NCone_s$ denotes bounded-fan-in Boolean circuits of size at most $s$ and depth
at most $c_{\NCone}\log(n+1)$, for a fixed constant $c_{\NCone}$.

\textbf{Input-output certificates.}
All certificates are over a fixed finite domain $\mathcal X$. Fix a hypothesis class $\mathcal H$ and a target $f^\star\in\mathcal H$. We record a labeled sample by its input set $S\subseteq\mathcal X$ where the corresponding labels attached to $x\in S$ are $f^\star(x)$. The version space after seeing $S$ is
\[
\VS_{f^\star,\mathcal H}(S)
:=
\{h\in\mathcal H:h(x)=f^\star(x)\text{ for every }x\in S\}.
\]
The set $S$ is an input-output certificate for $f^\star$ in $\mathcal H$ if
every surviving hypothesis agrees with $f^\star$ on the full domain:
\[
\forall h\in\VS_{f^\star,\mathcal H}(S),\qquad
h(x)=f^\star(x)\text{ for every }x\in\mathcal X.
\]
The certificate size is
\[
\cert(f^\star,\mathcal H)
:=
\min\Bigl\{
|S|:
S\subseteq\mathcal X
\text{ and } S \text{ is an input-output certificate for } f^\star
\text{ in } \mathcal H
\Bigr\}.
\]
Since $S=\mathcal X$ is a certificate, this minimum is well defined. The
definition is semantic. It certifies the input-output behavior of the target
and does not distinguish hypotheses that behave identically on $\mathcal X$.
This is the usual teaching/specification notion
\citep{shinohara1991teachability,goldman1995complexity,anthony1992exact}.
  
\textbf{Approximate certificates.}
Keep the same finite domain $\mathcal X$, hypothesis class $\mathcal H$, and
target $f^\star$. For $h\in\mathcal H$, define its absolute error count and
relative error rate with respect to $f^\star$ by
\[
\errabs(h):=|\{x\in\mathcal X:h(x)\ne f^\star(x)\}|,
\qquad
\errrel(h):=\frac{\errabs(h)}{|\mathcal X|}.
\]
For $\rho\in\{\abs,\rel\}$, the worst remaining error after seeing $S$ is
\[
M_\rho(S) := \max_{h\in\VS_{f^\star,\mathcal H}(S)}\err_\rho(h).
\]
For $\varepsilon,R\ge 0$, the approximate certificate sizes are
\[
\certrel_\varepsilon(f^\star,\mathcal H) := \min\{|S|:S\subseteq\mathcal X\text{ and }M_{\rel}(S)\le\varepsilon\},
\]
and
\[
\certabs_R(f^\star,\mathcal H)
:=
\min\{|S|:S\subseteq\mathcal X\text{ and }M_{\abs}(S)\le R\}.
\]

Thus $\certrel_\varepsilon$ certifies that every surviving hypothesis has uniform error at most $\varepsilon$, while $\certabs_R$ certifies that every surviving
hypothesis makes at most $R$ mistakes on $\mathcal X$. Exact certification is
the case $\varepsilon=0$, equivalently $R=0$.

%% file: 04_1_hardness.tex
\section{Certification hardness  from  overparametrization}
\label{sec:hardness}

Throughout this section, we lower-bound the number of labeled examples required to certify a target uniquely within an enlarged hypothesis class. The proof follows a combinatorial strategy to construct many alternative hypotheses around a fixed target so that each labeled example eliminates only a few of them. In our setting, the alternatives are deceivers supported on pairwise disjoint sets.

\subsection{Proof idea: trigger blocks and deceivers}
\label{sec:deceivers}

Fix a target $f^\star:\{0,1\}^n\to\{0,1\}$ in the hypothesis class.
Choose $t$ input coordinates $I\subseteq[n]$ and partition the domain according to their possible assignments, indexed by $T:=\{0,1\}^I$ with $|T|=2^t$.
Each assignment $\pi\in T$ defines a trigger block
$B_\pi:=\{x\in\{0,1\}^n:x_I=\pi\}$.
For each block, we build a deceiver $g_\pi$ that agrees with $f^\star$ outside $B_\pi$ and changes the output at least at one point inside it.
Since the trigger blocks are disjoint, a labeled example can eliminate at most one deceiver.
Therefore, any exact certificate must contain at least one point from every block, as formalized in the following proposition.

\begin{proposition}[Disjoint-disagreement principle]
\label{prop:block-elim}
Let $\mathcal X$ be a finite domain, let
$\mathcal H\subseteq\{h:\mathcal X\to\{0,1\}\}$, and let
$f^\star\in\mathcal H$. Let $T$ be a finite index set, and suppose that for each $\pi\in T$ there is a hypothesis $g_\pi\in\mathcal H$. Define
\[
E_\pi:=\{x\in\mathcal X:g_\pi(x)\ne f^\star(x)\}.
\]
If the sets $\{E_\pi:\pi\in T\}$ are nonempty and pairwise disjoint, then $\cert(f^\star,\mathcal H)\ge |T|.$
\end{proposition}

The full proof appears in \Cref{app:disjoint-block}, but the key point is that each labeled example can eliminate only the deceivers that disagree with $f^\star$ at that point, so disjoint disagreement sets force any certificate to hit every block.
In the following sections, we instantiate this principle for circuit classes and Transformers by showing that the enlarged hypothesis class contains exponentially many deceivers in the input dimension.

%% file: 04_2_circuits.tex
\subsection{Circuit classes}
\label{sec:circuits}

We instantiate the deceiver mechanism for circuits, focusing on threshold circuits where the construction is sharpest. Starting from threshold circuits with $s$ gates, increasing the gate budget to $s+1$ gates of the same depth suffices to override the output on a chosen block while preserving depth. We also state the corresponding results for the other circuit classes, deferring the details to \Cref{app:circuit-proofs}.

\begin{figure}[t]
    \centering
    \includegraphics[width=0.75\linewidth]{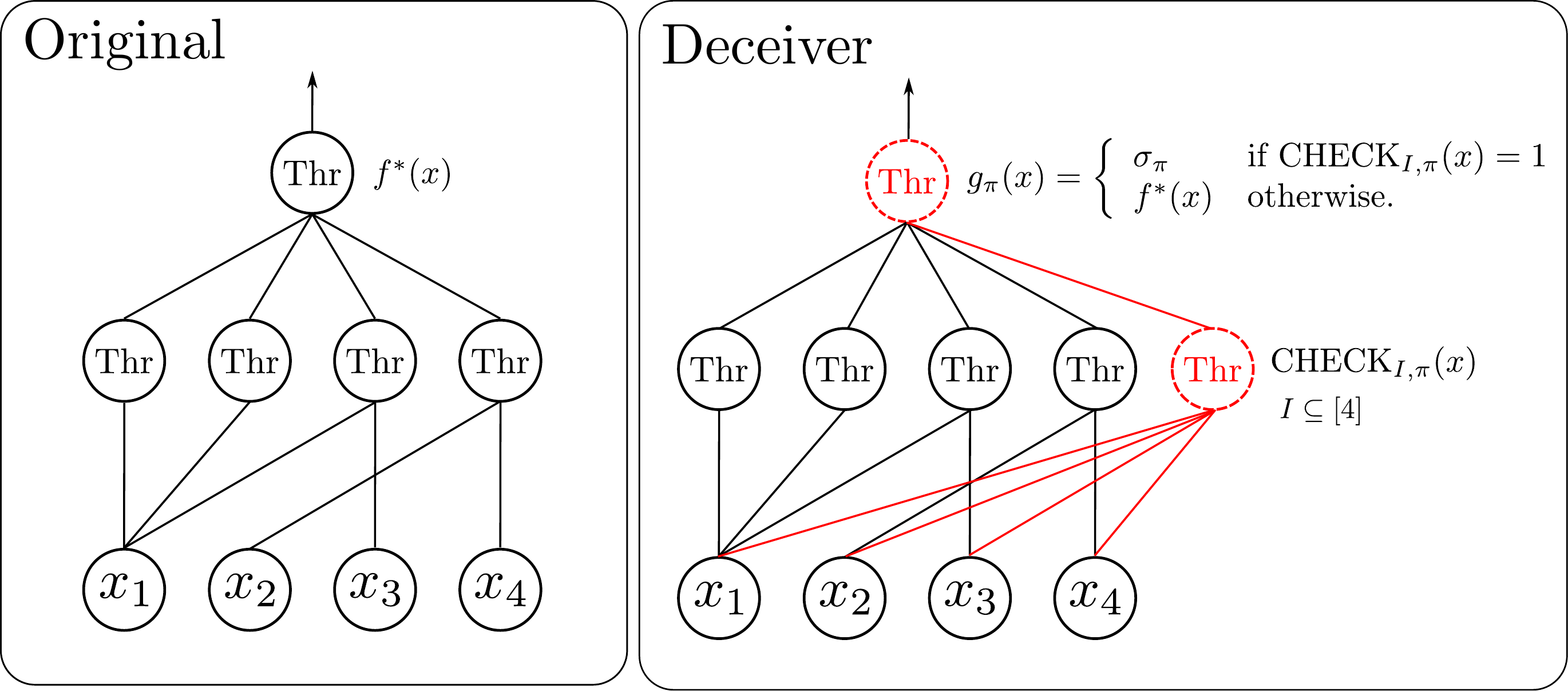}
    \caption{Same-depth deceiver construction in $\TCzero$. Left: the original threshold circuit computing $f^\star$. Right: one additional threshold gate $\mathrm{CHECK}_{I,\pi}$ detects the trigger block and feeds the top threshold, which forces the output on that block while leaving the circuit unchanged elsewhere.}
    \label{fig:deceiver-circuit}
\end{figure}

\begin{theorem}[One-gate certification hardness in $\TCzero$]
\label{thm:unbounded}
\label{thm:tc0-nonvacuous}
Fix a constant depth $d\ge 2$. For every $n$, every gate budget $s=n^{O(1)}$, every $f^\star\in\TCzero_{d,s}$, and every $1\le t\le n$,
\[
\cert\bigl(f^\star,\TCzero_{d,s+1}\bigr)\ge 2^t,
\]
consequently, for any trigger length $t=t(n)$ satisfying
$c n\le t(n)\le n$ for some constant $c>0$, the lower bound is $2^{\Omega(n)}$.
Moreover, for all sufficiently large $n$ there exists $f_n\in\TCzero_{d,s(n)}$ such that
\[
\cert\bigl(f_n,\TCzero_{d,s(n)}\bigr)\le n^{O(1)}
\quad\text{while}\quad
\cert\bigl(f_n,\TCzero_{d,s(n)+1}\bigr)\ge 2^{\Omega(n)}.
\]
\end{theorem}

\begin{proofsketch}

We first prove the lower bound. Fix a circuit of size at most $s$ and depth at most $d$ computing $f^\star\in\TCzero_{d,s}$. Choose $t$ input coordinates $I\subseteq[n]$ and partition the domain according to their possible $2^t$ binary assignments, forming $T:=\{0,1\}^I$. For each $\pi\in T$, let $B_\pi:=\{x\in\{0,1\}^n:x_I=\pi\}$ be the corresponding trigger block.

For each trigger block $B_\pi$, we construct a deceiver $g_\pi$ in two steps.
First, an additional threshold gate, denoted
$\mathrm{CHECK}_{I,\pi}$, detects whether the input belongs to $B_\pi$ by thresholding the corresponding signed input literals. Second, we connect this gate to the original output gate with a sufficiently large positive or negative weight. This forces a chosen constant output on $B_\pi$, while preserving
the original computation outside it.
The construction adds only one gate and keeps the depth at most $d$, as illustrated in \Cref{fig:deceiver-circuit}.

For each block $B_\pi$, choose a constant label
$\sigma_\pi\in\{0,1\}$ that disagrees with $f^\star$ on at least half of its inputs. The resulting deceiver
$g_\pi\in\TCzero_{d,s+1}$ agrees with $f^\star$ outside $B_\pi$ and disagrees with it on a nonempty subset inside it.
Since these disagreement sets depend on $B_\pi$'s which are disjoint, \Cref{prop:block-elim} gives $\cert(f^\star,\TCzero_{d,s+1})\ge 2^t$.
Taking $t=\Omega(n)$ yields an exponential lower bound.

To establish the second claim, we show that some target has a polynomial-size certificate in the smaller class $\TCzero_{d,s}$. A circuit with at most $s$ gates can be counted by observing that each gate computes one of at most $2^{O((n+s)^2)}$ threshold functions of the previously available Boolean signals. Hence $|\TCzero_{d,s}|\le 2^{O(s(n+s)^2)}$.
Every finite class contains some hypothesis with a certificate of size logarithmic in the class cardinality, by repeatedly keeping the smaller side of a split until one hypothesis remains. Therefore, when $s(n)=n^{O(1)}$, some $f_n\in\TCzero_{d,s(n)}$ satisfies $\cert(f_n,\TCzero_{d,s(n)})\le n^{O(1)}$. Applying the lower bound above to this same $f_n$, with $t=\Omega(n)$, gives $\cert(f_n,\TCzero_{d,s(n)+1})\ge 2^{\Omega(n)}$.
\end{proofsketch}

\textbf{Certification hardness in other circuit classes:} The deceiver construction has two parts: detecting a trigger block and overriding the output on that block without changing the computation elsewhere. While in $\TCzero$, both have constant cost, the deceiver constructions for $\ACzero$ and $\NCone$, which are described in \Cref{app:circuit-proofs}, require extra components. In particular, for $\ACzero$, trigger detection still costs one gate, but the override function requires an additional combining gate and typically one extra depth level. Thus targets in $\ACzero_{d,s}$ satisfy the same $2^t$ certification bounds inside $\ACzero_{d+1,s+2}$. In bounded fan-in circuits such as $\NCone$, exact trigger detection costs $\Theta(t)$ extra gates and $\Theta(\log t)$ extra depth, so $k$ extra gates give lower bounds of $2^{\Omega(\min\{n,k\})}$.

%% file: 04_3_transformers.tex
\subsection{Transformers}
\label{sec:transformers}
The circuit computations described so far also underlie many reasoning tasks studied in Transformers. Beyond this task connection, there is a stronger relationship between Transformers and circuits: existing expressivity results show that log-precision Transformer computations can be simulated by threshold circuits \citep{merrill2022saturated,merrill2023logic,merrill2024expressive,chiang2025transformers}.
However, these results are one-way, as they do not imply that arbitrary threshold-circuit computations can be realized by Transformers. We therefore prove the certification barrier directly inside a Transformer model.

We consider finite-precision Transformers with unmasked
averaging-hard attention (AHAT). Every position can attend to
every other position, and the values at all maximum-scoring
positions are averaged. Each layer combines attention and
feedforward updates with residual connections. The representation
at the final position is mapped to a scalar whose sign determines
the binary output. Finite precision restricts the possible
parameter values, which bounds the size of the hypothesis class.

For a fixed input length $n$, precision $p$, and architecture
$\mathsf A_n$ specifying the number of layers, attention heads,
and embedding/residual dimensions, the class $\mathcal T_{\mathsf A_n,p}(n)$ contains all Boolean functions computed by admissible parameter values. The complete definition, including the normalization,
admissibility conditions, and finite-precision arithmetic,
is given in \Cref{def:general-ahat-class}.

\begin{theorem}[Transformer certification hardness under overparametrization]
\label{thm:transformer-certification}
There is an absolute precision constant $p_0$ such that the following holds. Fix an input length $n$, an AHAT architecture $\mathsf A_n$, and a precision $p\ge p_0$. Let $1\le t\le n$ and $\mathsf A_n^+$ be the extension of $\mathsf A_n$ specified in \Cref{app:thm:transformer-deceiver}, which adds one attention head to the final layer and six coordinates to the residual stream, without increasing the number of layers. Then every target computed by the original architecture is hard to certify in the enlarged one:
\[
\forall f^\star\in\mathcal T_{\mathsf A_n,p}(n),\qquad
\cert\bigl(f^\star,\mathcal T_{\mathsf A_n^+,p}(n)\bigr)\ge 2^t .
\]
Moreover, this lower bound is not vacuous. If $\{\mathsf A_n\}$ is a fixed-architecture AHAT family with $n^{O(1)}$ finite-precision scalar parameters and $p(n)=O(\log n)$, assume there exists $n_0$ such that $p(n)\ge p_0$ for all $n\ge n_0$. Then for every $n\ge n_0$ there is a function
$f_n\in\mathcal T_{\mathsf A_n,p(n)}(n)$ such that
\[
\cert\bigl(f_n,\mathcal T_{\mathsf A_n,p(n)}(n)\bigr)\le n^{O(1)}
\quad\text{while, taking any } t=\Omega(n)
,\quad
\cert\bigl(f_n,\mathcal T_{\mathsf A_n^+,p(n)}(n)\bigr)\ge 2^{\Omega(n)}.
\]
\end{theorem}

\begin{proofsketch}
Fix $f^\star\in\mathcal T_{\mathsf A_n,p}(n)$ and a Transformer
$T^\star$ that computes it. As in \Cref{thm:unbounded}, the domain is partitioned into $2^t$ blocks $B_\pi$ by fixing $t$ input coordinates.

For each block, we construct a deceiver using one additional
final-layer attention head. Four additional embedding coordinates are used to encode whether each position matches, mismatches, or lies outside the selected coordinates. 
The additional head uses hard attention to indicate if the given input is in $B_\pi$.
Two extra coordinates carry this bit to the final readout,
which forces a hardcoded output $\sigma_\pi \in \{0,1\}$ for inputs in $B_\pi$ while preserving the original output elsewhere. The inherited computation remains unchanged, with the same depth and precision.
See \Cref{app:thm:transformer-deceiver} for details.

Choosing $\sigma_\pi$ to disagree with $f^\star$ on at least half of $B_\pi$ gives disjoint, nonempty disagreement sets.
Thus, \Cref{prop:block-elim} yields $\cert(f^\star,\mathcal T_{\mathsf A_n^+,p}(n))\ge 2^t$.

The non-vacuity claim follows from \Cref{app:prop:poly-cert-transformer}, which guarantees an easily certifiable target in a nonempty, polynomial-parameter, log-precision class. Applying the lower bound to this target with $t=\Theta(n)$ gives the exponential separation.
\end{proofsketch}

Thus the Transformer result is not just a statement about a function that was already hard to certify. The smaller architecture contains polynomially certifiable functions, but after adding a single extra head and constant residual width, the same function can require exponentially many labeled examples to certify against the enlarged Transformer class.

%% file: 04_4_approximate.tex
\subsection{From exact to approximate certification}

The results above concern exact certification. One may ask whether hardness persists under a more forgiving notion, where a consistent hypothesis is only required to be approximately correct. 
The answer depends sharply on how the tolerance is measured.
If only polynomially many mistakes are allowed, the block-deceiver construction still forces certificates of exponential size.
In contrast, a fixed relative tolerance can leave exponentially many small-error block deceivers unresolved, at the cost of permitting exponentially many mistakes overall.
The following proposition captures this general argument, and it applies whenever the hypothesis class admits the required block deceivers, as in the circuit and Transformer constructions of the previous sections.

\begin{proposition}[Approximate certification lower bounds from block deceivers]
\label{prop:block-approx}
Let $\mathcal X = \{0,1\}^n$ and fix a set $I \subseteq [n]$ of $t$ trigger coordinates. For each pattern $\pi \in \{0,1\}^I$, let $B_\pi := \{x \in \mathcal X : x_I = \pi\}$ be the corresponding trigger block. Suppose that for every pattern $\pi$, there exists a deceiver $g_\pi \in \mathcal H$ that agrees with $f^\star$ outside $B_\pi$ and disagrees with it on at least half of $B_\pi$, i.e., writing $E_\pi := \{x : g_\pi(x) \ne f^\star(x)\}$,
\[
    E_\pi \subseteq B_\pi, \qquad |E_\pi| \ge 2^{n-t-1}.
\]
Then:
\[
    \varepsilon_n < 2^{-t-1} \;\Rightarrow\; \certrel_{\varepsilon_n}(f^\star, \mathcal H) \ge 2^t
    \qquad\text{and}\qquad
    R_n < 2^{n-t-1} \;\Rightarrow\; \certabs_{R_n}(f^\star, \mathcal H) \ge 2^t.
\]
\end{proposition}

\begin{proof}
The sets $E_\pi$ are nonempty and pairwise disjoint, since each $E_\pi\subseteq B_\pi$ and the blocks partition $\mathcal X$. Suppose a labeled set $S$ misses some $E_\pi$. Then $g_\pi\in\VS_{f^\star, \mathcal H}(S)$, with
\[
    \errrel(g_\pi) = 2^{-n}|E_\pi| \ge 2^{-t-1},
    \qquad
    \errabs(g_\pi) = |E_\pi| \ge 2^{n-t-1}.
\]
If $\varepsilon_n<2^{-t-1}$ or $R_n<2^{n-t-1}$, this deceiver violates the respective guarantee. Hence any certificate must hit every $E_\pi$, and therefore has size at least $2^t$.
\end{proof}

The lower bound in \Cref{prop:block-approx} increases as the tolerance decreases, up to $2^n$.
However, for a fixed relative tolerance, the bound does not grow with $n$.

\textbf{Certifying near-perfect accuracy is exponentially expensive:}
If $R_n = \operatorname{poly}(n)$, then choosing $t = n - O(\log n)$ gives 
\[
    \certabs_{R_n}(f^\star, \mathcal H) \ge \frac{2^n}{\operatorname{poly}(n)}    = 2^{\Omega(n)}.
\]
Thus, allowing only polynomially many mistakes still forces the certificate size to be exponential to cover all deceivers that have more than $R_n$ mistakes.

\textbf{A fixed relative tolerance allows deceivers with exponentially many mistakes to survive:}
For fixed $\varepsilon>0$, the block construction no longer
yields an exponential certificate-size lower bound. In fact, a standard greedy argument gives $\certrel_\varepsilon(f^\star,\mathcal H) =O(\log|\mathcal H|/\varepsilon)$ for finite classes,
implying polynomial-size certificates for the hypothesis classes considered here. Nevertheless, a fixed relative-error guarantee still allows up to $\varepsilon 2^n$ mistakes, which is exponential in $n$.

%% file: 05_experiments.tex
\section{Experiments}
\label{sec:experiments}

We complement the theoretical results with two experiments on binary-addition recognition. The circuit experiment instantiates our lower-bound construction. The Transformer experiment examines how errors are distributed across models that pass strict training and validation checks, and how difficult these errors are to detect through sampling.

Each input is a triple $(a,b,z)$, where $a,b\in\{0,1\}^n$ are operands and $z\in\{0,1\}^{n+1}$ is a proposed sum. The target is $f^\star(a,b,z)=\mathbf 1[z=a+b]$. Here $n$ denotes the bit length of the operands, while the full Boolean input length is $3n+1$.

Both experiments report the expected number of deceivers that remain consistent with uniformly sampled labeled examples. Sampling is done without replacement, and a deceiver survives if none of its errors is sampled. The calculations are given in \Cref{app:circuit-experiments,app:addition-survivors}.

\begin{figure}[t]
\centering
\includegraphics[width=0.49\linewidth]{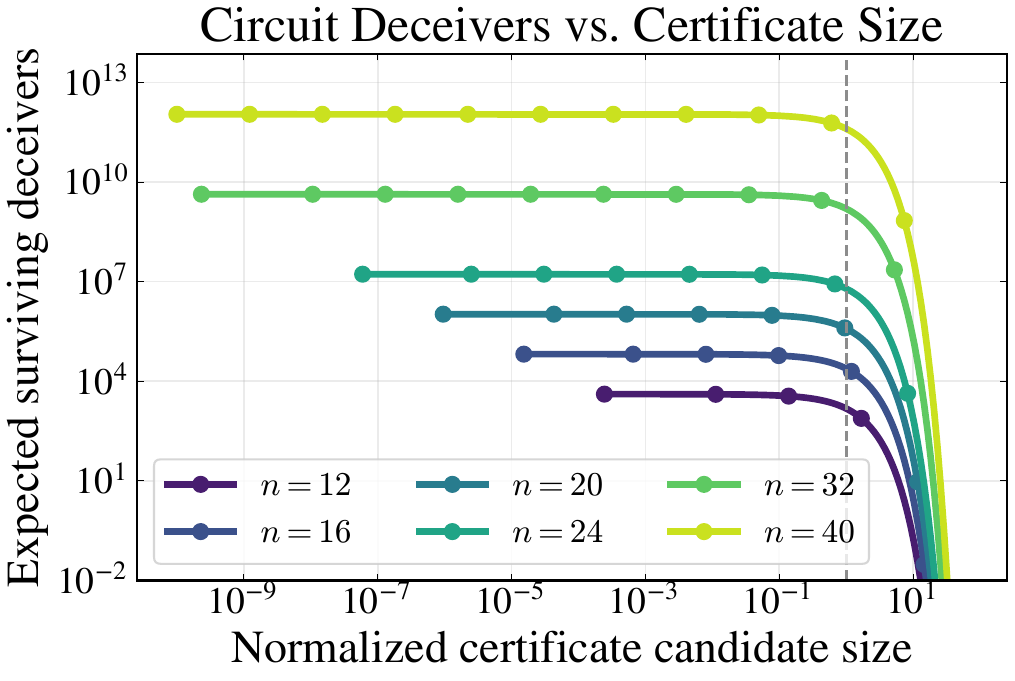}
\hfill
\includegraphics[width=0.49\linewidth]{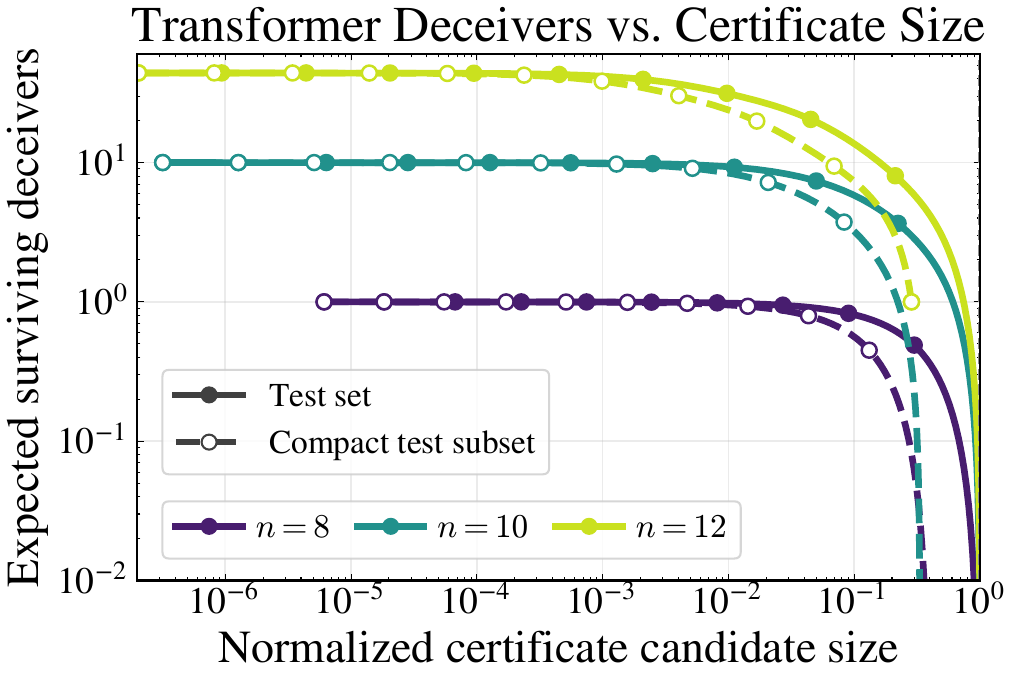}
\caption{
Expected number of surviving deceivers for binary-addition recognition under uniform sampling without replacement.
\textbf{Left.} Constructed threshold circuits, sampled from the full input domain, with sample size $m$ normalized by $2^n$.
\textbf{Right.} Trained Transformers at $n=8,10,12$. Solid curves show sampling from the full held-out test set, while dashed curves use correct sums and one-bit flips of the three highest output bits. Sample sizes are normalized by the full test-set size $N_{\rm test}$ at each $n$. Both axes use logarithmic scales, with zero values omitted.
}
\label{fig:addition-side-by-side}
\end{figure}

\paragraph{Constructed threshold circuits.}
For addition, the trigger construction of \Cref{thm:unbounded} starts from a target recognizer with two threshold gates at depth $2$, the minimum possible. For each pattern $\pi\in\{0,1\}^n$, adding one gate gives a deceiver that accepts every input in $B_\pi=\{(a,b,z):a=\pi\}$ and agrees with the target elsewhere. The resulting $2^n$ deceivers use three gates at depth $2$, and their error sets are nonempty and pairwise disjoint. Every certificate against this three-gate class therefore requires at least $2^n$ labeled examples. Consequently, any $\poly(n)$ certificate candidate leaves a fraction of deceivers tending to one as $n$ grows. The left panel of \Cref{fig:addition-side-by-side} shows their expected survival when sampling uniformly from the full domain $\mathcal X=\{0,1\}^{3n+1}$. The construction and exact calculation are given in \Cref{app:circuit-experiments}.

\paragraph{Trained Transformer recognizers.}
We train $150$ encoder-only Transformers at each bit length $n\in\{8,10,12\}$ to recognize correct additions. A fixed split assigns $75\%$ of operand pairs $(a,b)$ to training and validation and $25\%$ to testing. Intermediate supervision provides the correct sum bits, intermediate sums, and carries during training. At evaluation, each model receives only $(a,b,z)$ and returns an accept/reject prediction.

We use fixed training and validation checks, each containing up to $20{,}000$ summand pairs.
In each check, we verify the correct sum $z=a+b$ as well as its $n+1$ one-bit flips. Once models pass both checks they are then evaluated on held-out pairs $(a,b)$ using same candidate-sum construction. We call a model \emph{test-exact} if it makes no errors on the entire test set.\footnote{This is exactness on the specified held-out set, not over all possible triples $(a,b,z)$.} Exact check sizes and training details are given in \Cref{app:transformers-addition}.

Among the $150$ models at each bit length, $1$, $10$, and $44$
have held-out errors at $n=8,10,12$, respectively. Among these
incorrect models, the median numbers of errors are $2$, $5$,
and $15.5$, respectively. We call these incorrect models \emph{deceivers}, although their errors arise from training rather than from an inserted trigger like in our proofs.

The relevant observation for certification is how much their error sets overlap. A labeled example can eliminate more than one deceiver only if those models make a mistake on the same input.
At $n=10$, all ten deceivers have pairwise-disjoint error sets in the held-out test set.
At $n=12$, from the $44$ deceivers, $40$ have pairwise-disjoint error sets. 
Consequently, a certificate in this case must individually cover almost all deceivers found.

Another obstacle to certification is that these examples must also be found. When their locations are unknown, one approach is to sample uniformly without replacement from the test set.
The right panel of \Cref{fig:addition-side-by-side} shows the expected number of surviving deceivers as the sample size increases.
At $n=10$, ten selected labeled examples can eliminate all ten deceivers. In contrast, uniform sampling from the same held-out pool instead requires $3{,}010{,}601$ examples to reach a $95\%$ probability of detecting all ten. The joint-detection calculation is given in \Cref{app:addition-all-models}.

One potential objection is that the test set might be too broad, with errors concentrated in a smaller region that a more targeted sampling strategy could exploit. We therefore consider a structured subset containing, for each held-out pair, the correct sum $z=a+b$ and its three one-bit flips in the highest output positions. This reduces the number of candidate sums per pair from $n+2$ to $4$, while retaining $100\%$, $90.3\%$, and $80.0\%$ of the distinct error locations at $n=8,10,12$, respectively. Sampling from this subset improves detection, as shown by the dashed curves in the right panel of \Cref{fig:addition-side-by-side}. Even so, eliminating all deceivers with errors in the subset by uniform sampling requires a substantial fraction of its examples. At $n=10$, the subset contains $1{,}048{,}576$ examples. Moreover, as shown by the endpoint of the $n=12$ dashed curve, one deceiver has no errors in this subset and cannot be eliminated by sampling from it.

%% file: 06_conclusion.tex
\section{Conclusion and future work}
\label{sec:discussion}

We studied certification from labeled input-output examples for circuits and log-precision Transformers. Our main finding is that certification is highly sensitive to the ambient hypothesis class, with even minimal overparametrization, such as one threshold gate or one final-layer attention head with constant residual-width overhead, creating exponentially many trigger deceivers around every target in the smaller class. Hence, polynomially many examples may fail to certify that a model has learned the intended computation.
For models performing algorithmic tasks or reasoning, this highlights the gap between high accuracy on sampled inputs and a guarantee of correct behavior on all inputs.

An important next step is to understand which stronger certification protocols can escape this barrier. 
A natural direction is active certification with access to richer feedback, such as counterexamples from a trusted verifier. The exact learning literature shows that membership and equivalence queries can change learnability and query complexity \citep{angluin1987learning,hellerstein1996queries}. For neural networks, especially Large Reasoning Models, such feedback or access to intermediate computations could help select inputs adaptively and detect errors that are difficult to identify from final outputs alone. Another question is whether the biases of the training process restrict the hypotheses that arise, making certification easier than the worst-case bounds for the full class suggest. Finally, extending certification to randomized or unreliable threshold circuits \citep{hajnal1993threshold} is another important direction.

%% file: 990_basic.tex
\section{Generic certification tools}
\label{app:cert-tools}

\subsection{Disjoint-block principle}
\label{app:disjoint-block}

\begin{proposition}[Restatement of \Cref{prop:block-elim}]
\label{app:prop:block-elim}
Let $\mathcal X$ be a finite domain, let
$\mathcal H\subseteq\{h:\mathcal X\to\{0,1\}\}$, and let
$f^\star\in\mathcal H$. Let $T$ be a finite index set, and suppose that for
each $\pi\in T$ there is a hypothesis $g_\pi\in\mathcal H$. Define
\[
E_\pi:=\{x\in\mathcal X:g_\pi(x)\ne f^\star(x)\}.
\]
If the sets $\{E_\pi:\pi\in T\}$ are nonempty and pairwise disjoint, then $\cert(f^\star,\mathcal H)\ge |T|.$
\end{proposition}

\begin{proof}[Proof]
Let $S\subseteq\mathcal X$ be any labeled sample consistent with $f^\star$.
A deceiver $g_\pi$ is eliminated by $S$ only if $S$ contains a point in $E_\pi$.
Since the sets $\{E_\pi:\pi\in T\}$ are pairwise disjoint, each labeled example eliminates at most one deceiver. If $|S|<|T|$, then some $\pi\in T$ has $S\cap E_\pi=\emptyset$, so $g_\pi$ remains consistent with all labels in $S$ while $g_\pi\ne f^\star$.
Hence $S$ is not a certificate.
Every certificate must therefore have size at least $|T|$, i.e.,
$\cert(f^\star,\mathcal H)\ge |T|$.
\end{proof}

In the later sections, the circuit and Transformer lower bounds all reduce to this proposition. In each case, we construct hypotheses $g_\pi$ whose disagreement sets are nonempty and contained in pairwise disjoint trigger blocks $B_\pi$, instantiated below for a binary domain.

\textbf{Trigger blocks:} For a binary input space of $n$ coordinates, fix a coordinate set $I \subseteq [n]$ with $|I| = t$. For each pattern $\pi \in \{0,1\}^I$, define the trigger block
\[
B_\pi := \{x \in \{0,1\}^n : x_i = \pi_i \text{ for all } i \in I\}.
\]
The blocks $\{B_\pi\}_{\pi \in \{0,1\}^I}$ are pairwise disjoint. In the instantiations considered for the hypothesis classes covered here, it is useful to view a deceiver $g_\pi$ as a function that agrees with $f^\star$ outside $B_\pi$ and is fixed to a constant value on $B_\pi$.

\subsection{Lower bounds on certificate size in finite classes}

Our exponential certification lower bounds do not imply that the underlying hypothesis class is uniformly hard. In fact, every finite hypothesis class contains at least one hypothesis
with a certificate of size $O(\log|\mathcal H|)$.

\begin{lemma}[A small certificate exists in every finite class]
\label{app:lem:small-cert-exists}
Let $\mathcal X$ be a domain and let
$\mathcal H\subseteq \{0,1\}^{\mathcal X}$ be finite and nonempty.
Then there exists $f\in\mathcal H$ such that
\[
\cert(f,\mathcal H)\le \lceil \log_2 |\mathcal H|\rceil.
\]
\end{lemma}

\begin{proof}
Initialize $V_0:=\mathcal H$ and an empty input set $S$.
While $|V_t|>1$, choose some $x_t\in\mathcal X$ on which two hypotheses in
$V_t$ disagree.  Let
\[
V_t^\pi:=\{h\in V_t:h(x_t)=\pi\}
\qquad \pi\in\{0,1\}.
\]
Both parts are nonempty, and the smaller one has size at most $|V_t|/2$.
Choose $\pi_t$ with $|V_t^{\pi_t}|\le |V_t|/2$, add $x_t$ to $S$, and set
$V_{t+1}:=V_t^{\pi_t}$.  The intended label of $x_t$ is $\pi_t$. After at most $\lceil\log_2|\mathcal H|\rceil$ steps we obtain a singleton
$V_T=\{f\}$.  By consatruction, $f$ is the unique hypothesis in $\mathcal H$
consistent with the labels prescribed on $S$, so $S$ is a certificate for
$f$ of size at most $\lceil\log_2|\mathcal H|\rceil$.
\end{proof}

The results in this subsection are best-case upper-bound statements. They prove that for a finite hypothesis class $\mathcal H$,
\[
\exists h\in\mathcal H:\quad \cert(h,\mathcal H)\le O(\log|\mathcal H|).
\]
The lower bounds in \Cref{sec:circuits} have a different quantifier order.
For example, the threshold-circuit result proves that for every
$f^\star\in\TCzero_{d,s}$, the certificate size in the larger class $\cert(f^\star,\TCzero_{d,s+1})$ is exponential.
There is no contradiction. The small certificate size guaranteed inside the enlarged class need not lie in the smaller target class. Thus the same class can contain large and very small certificate sizes for different targets.

%% file: 991_circuits.tex
\section{Circuit constructions and proofs}
\label{app:circuit-proofs}

\subsection{Certificate size bounds from circuit counting}

In the following, we instantiate the principle of \Cref{app:lem:small-cert-exists}, showing that there exist hypotheses with polynomial certificate size across our different circuit classes.

\begin{proposition}[Existence of a polynomial certificate size in $\TCzero_{d,s}$]
\label{app:prop:tc0-exists-poly-cert}
Fix $n\in\N$, depth $d\ge 1$, and $s\ge 1$, and let $\mathcal H := \TCzero_{d,s}$ be the semantic class of Boolean functions computable by depth-$d$ threshold circuits with $n$ external inputs and at most $s$ gates.
Then there exists $f\in\mathcal H$ with
\[
\cert(f,\mathcal H) \le O\bigl(s\,(n+s)^2\bigr).
\]
In particular, if $d$ is fixed and $s=n^{O(1)}$, then $\mathcal H$ contains some function with polynomial certificate size.
The same holds for the one-gate enlargement $\TCzero_{d,s+1}$ by substituting $s\leftarrow s+1$.
\end{proposition}

\begin{proof}
For fixed $n$, $\mathcal H$ is finite (it is a set of Boolean functions on $\{0,1\}^n$). Let $T(k)$ denote the number of Boolean linear threshold functions on $k$ inputs. Each input $x\in\{0,1\}^k$ defines a central hyperplane $w\cdot x-\theta=0$ in the $(k+1)$-dimensional parameter space. Every threshold function on this finite domain has a representation with no input on the boundary, by a sufficiently small shift of $\theta$. Thus it is counted by a region of this arrangement. The hyperplane-arrangement bound \citep{baldi2019polynomial} gives
\[
T(k)\le 2\sum_{j=0}^{k}\binom{2^k-1}{j}\le 2^{O(k^2)}.
\]
The square is the product of $\log_2(2^k)=k$ and the order-$k$ parameter-space dimension. To canonize each entry, we may delete unused gates and, if necessary, pad with dummy gates. Then, we topologically order the $s$ gates. Since gate $i$ depends on at most
$n+i-1\le n+s$ previously available outputs, it has $T(n+s)$
possible Boolean behaviors, hence
\[
|\mathcal H|\le T(n+s)^s\le 2^{O(s(n+s)^2)}.
\]
The claim follows by applying \Cref{app:lem:small-cert-exists}.
\end{proof}

\begin{proposition}[Existence of a polynomial certificate size in $\ACzero_{d, s}$]
\label{app:prop:poly-cert-ac0}
Fix $n,s,d\in\N$ and let $\mathcal H:=\ACzero_{d,s}$ be the semantic class of depth-$d$ unbounded-fan-in Boolean circuits with $n$ external inputs and at most $s$ gates.
Then there exists $f\in\mathcal H$ such that
\[
\cert\bigl(f,\mathcal H\bigr)\le O\bigl(s(n+s)\bigr).
\]
In particular, for constant $d$ and $s=n^{O(1)}$, the class contains at least one function with polynomial certificate size.
\end{proposition}

\begin{proof}
Count circuits under the free-literal convention from \Cref{sec:prelim} and order the $s$ internal gates.
At gate $i$, the available signals consist of the $2n$ input literals and at most $i-1$ previous gate outputs.
An AND/OR gate chooses a subset of these signals and one of two types. Since $i-1\le s-1$, there are at most $2\cdot2^{2n+s-1}=2^{2n+s}$ choices per gate. Therefore $|\mathcal H|\le2^{s(2n+s)}$, and \Cref{app:lem:small-cert-exists} gives the claim.
\end{proof}

\begin{proposition}[Existence of a polynomial certificate size in bounded-size fan-in-$2$ circuit classes]
\label{app:prop:poly-cert-nc1}
Fix $n,s,d\in\N$, and let $\mathcal H$ be the semantic class of Boolean
functions computable by fan-in-$2$ Boolean circuits with $n$ external inputs,
size at most $s$, and depth at most $d$. Then there exists $f\in\mathcal H$
such that
\[
\cert(f,\mathcal H)\le O\bigl(s\log(n+s)\bigr).
\]
Moreover, for every fixed $c>0$, the same conclusion holds for $\NCone_{s}$, the class of fan-in-$2$ circuits with $n$ inputs, size
at most $s$, and depth at most $c\log n$. In particular, if $s=n^{O(1)}$,
then $\NCone_{s}$ contains a function with polynomial certificate size.
\end{proposition}

\begin{proof}
We count the circuits after topologically
ordering the $s$ gates, where gate $i$ may read from at most $n+i-1\le n+s$
previously available signals. Thus each gate has at most $3(n+s)^2$ possible
descriptions, accounting for $\wedge,\vee,\neg$ gates. Hence $|\mathcal H|\le (3(n+s)^2)^s$ and applying \Cref{app:lem:small-cert-exists} finishes the proof. 
The same argument applies to the subclass
$\NCone_{s}\subseteq \mathcal H$ with $d=\lfloor c\log n\rfloor$.
If $s=n^{O(1)}$, the bound is polynomial in $n$.
\end{proof}

\subsection{An explicit quadratic certificate in $\TCzero_{2,2}$}
\label{app:circuit}

In contrast to the existential proofs of the previous section, we give an explicit Boolean function with an optimal quadratic-size
certificate in $\TCzero_{2,2}$.
The target requires two gates and depth two, and the certificate is valid against all circuits in
the class.

\begin{proposition}[Exact certificate for a two-gate threshold circuit]
\label{imp:tc:prop:tc0-explicit-genuine-depth-two}
For $n\ge4$, write $|x|=\sum_{i=1}^n x_i$ and define
\[
 f_n(x)=\mathbf 1[\,|x|\in\{1,n\}\,].
\]
The minimum threshold-circuit size and depth of $f_n$ are both exactly two.
Moreover,
\begin{align*}
 \cert(f_n,\TCzero_{2,2})&=\binom n2+2n+2.
\end{align*}
Here the competing circuits may have arbitrary integer weights and arbitrary
connections from external inputs directly to the output gate. In particular,
competitors need not be symmetric.
The unique inclusion-minimal certificate in $\TCzero_{2,2}$ is
\[
 S_n=\{0^n,1^n\}
 \cup\{e_i:i\in[n]\}
 \cup\{e_i+e_j:1\le i<j\le n\}
 \cup\{1^n-e_i:i\in[n]\}.
\]
\end{proposition}

\begin{proof}
The two-gate circuit
\begin{equation}
 f_n(x)=\mathbf 1[\,|x|-(n-1)q(x)\ge1\,] \quad\textrm{ with } \quad q(x)=\mathbf 1[\,|x|\ge2\,].
 \label{imp:tc:eq:one-all-realization}
\end{equation}
accepts exactly the inputs of weights one and $n$. 

The first step is to show that no single threshold gate computes $f_n$.  To see that, suppose $f_n(x)=\mathbf 1[A(x)\ge0]$ for an affine function $A(x)=a_0+\sum_i a_ix_i$.
For distinct $i,j$, the target labels require
$A(e_i),A(e_j)\ge0$ and $A(0),A(e_i+e_j)<0$ but this contradicts the identity $A(e_i)+A(e_j)=A(0)+A(e_i+e_j)$.
Therefore, $f_n$ requires at least two gates and depth two as shown above.

\textbf{Sufficiency of $S_n$.}
Recall that $S_n$ contains all inputs of Hamming weights
$0,1,2,n-1,n$. Consider any $g\in\TCzero_{2,2}$
agreeing with $f_n$ on $S_n$. Every such circuit has the form
\begin{equation}
g(x)=\mathbf 1[A(x)+\lambda h(x)\ge0],\qquad
A(x)=a_0+\sum_{i=1}^n a_i x_i,\qquad
h(x)=\mathbf 1[c_0+\sum_{i=1}^n c_i x_i\ge0].
\label{imp:tc:eq:two-gate-normal-form}
\end{equation}
This allows direct input connections to the output gate.
The one-gate impossibility argument rules out $\lambda=0$.
If $\lambda<0$, we may complement $h$ and absorb the
resulting constant into $A$, so assume $\lambda>0$.

Write $h_0=h(0)$, $h_i=h(e_i)$, and
$h_{ij}=h(e_i+e_j)$. Since $g$ is positive at the
singletons and negative at zero and the pairs, the affine
terms cancel when comparing their scores, giving
\begin{equation}
h_i+h_j-h_0-h_{ij}>0
\qquad(i\ne j).
\label{imp:tc:eq:hidden-trade}
\end{equation}
We distinguish two cases.

\smallskip
\noindent\emph{Case 1: $h_0=1$.}
Equation~\eqref{imp:tc:eq:hidden-trade} forces
$h_i=1$ and $h_{ij}=0$ for all distinct $i,j$.
The hidden threshold inequalities then imply $c_i<0$
for every $i$, so $h(x)=0$ whenever $|x|\ge2$.
Moreover, comparing the output scores at zero and
at each singleton gives $a_i>0$.

For any $2\le|x|\le n-1$, choose a coordinate $j$
with $x_j=0$. Since $x\le1^n-e_j$ coordinatewise
and both hidden values are zero,
\[
A(x)\le A(1^n-e_j)<0.
\]
Thus $g(x)=0$. All remaining inputs are already
labeled by $S_n$, so $g=f_n$.

\smallskip
\noindent\emph{Case 2: $h_0=0$.}
Equation~\eqref{imp:tc:eq:hidden-trade} implies
that at most one singleton has hidden value zero.
If all had value one, then $c_i>0$ for every $i$,
so $h(x)=1$ at every nonzero input.
The positive singleton and negative pair labels
would imply $a_i<0$ for all $i$, contradicting
the positive label at $1^n$.

Hence there is a unique $k$ with $h_k=0$.
For every $i\ne k$, the singleton and pair labels give
\[
c_i>0,\qquad a_i<0,\qquad h(e_k+e_i)=0.
\]
On the face $x_k=0$, the hidden gate equals one
at every nonzero input. Every input of weight at least
two contains a negative pair, and adding coordinates
only decreases its output score. Thus $g$ agrees
with $f_n$ on this face.

On the face $x_k=1$, let $v=1^n$. Since $a_i<0$,
we have $A(v-e_i)>A(v)$ for every $i\ne k$.
The labels $g(v)=1$ and $g(v-e_i)=0$ therefore
force $h(v)=1$ and $h(v-e_i)=0$.
Because $c_i>0$ for $i\ne k$, the hidden gate
is increasing in these coordinates and hence
vanishes at every other input on this face.

Any such input of weight at least two contains
some $e_k+e_i$, where both the hidden value
and output are zero. Adding further coordinates
only decreases $A$, so these inputs remain negative.
The remaining points $e_k$ and $v$ are labeled
positive. Thus $g=f_n$ also in this case.

We have shown that $S_n$ is a certificate.
Since its five Hamming weights are distinct for $n\ge4$,
\[
|S_n|=1+n+\binom n2+n+1=\binom n2+2n+2.
\]

\textbf{Every example in $S_n$ is necessary.}
For each $z\in S_n$, we construct a two-gate function
$g_z$ differing from $f_n$ only at $z$.
Recall $q(x)=\mathbf 1[|x|\ge2]$.
For zero and the all-ones input, take
\begin{align*}
g_0(x)&=\mathbf 1[|x|-nq(x)\ge0],\\
g_{1^n}(x)&=\mathbf 1[|x|-nq(x)\ge1].
\end{align*}
These respectively add $0^n$ and remove $1^n$
from the positive set.

For singletons and their complements, take
\begin{align*}
g_{e_i}(x)
&=\mathbf 1[2|x|-x_i-(2n-3)q(x)\ge2],\\
g_{1^n-e_i}(x)
&=\mathbf 1[2|x|-x_i-(2n-3)q(x)\ge1].
\end{align*}
The first changes only the label at $e_i$:
among weight-one inputs, its score falls below
the threshold only there. The second adds
$1^n-e_i$, the only new input of weight at
least two reaching its threshold.

Finally, for each pair $e_i+e_j$, define
\begin{align*}
q_{ij}(x)&=\mathbf 1[2|x|-x_i-x_j\ge3],\\
g_{e_i+e_j}(x)
&=\mathbf 1[|x|-(n-1)q_{ij}(x)\ge1].
\end{align*}
The gate $q_{ij}$ agrees with $q$ everywhere
except at $e_i+e_j$, where its output changes
from one to zero. Hence $g_{e_i+e_j}$
differs from $f_n$ exactly at that pair.

Each constructed $g_z$ has at most two threshold gates
and depth at most two. Since omitting any
$z\in S_n$ leaves $g_z$ consistent with the
remaining examples, every certificate must contain
$S_n$. Together with sufficiency, this proves the
uniqueness and exact certificate size.
\end{proof}

\subsection{Certification hardness in $\TCzero$}
\label{app:tc0-overrides}

\begin{proposition}[Same-depth threshold-circuit deceivers]
\label{prop:tc0-compat}
Let $C^\star$ be a threshold circuit of size $s$ and depth at most $d$
computing $f^\star$, where $d \ge 2$.
Fix $I \subseteq [n]$ with $|I|=t$ and a pattern $\pi\in \{0,1\}^I$.
Then:
\begin{enumerate}[label=(\roman*)]
\item there is a threshold gate $\mathrm{CHECK}_{I,\pi}$ of depth $1$ such that
\[
\mathrm{CHECK}_{I,\pi}(x)=1 \iff x \in B_\pi;
\]

\item for each $\sigma \in \{0,1\}$ there is a threshold circuit
$C_{\pi,\sigma}$ of size at most $s+1$ and depth at most $d$ such that
\[
g_{\pi,\sigma}(x):=
\begin{cases}
\sigma, & x \in B_\pi,\\
f^\star(x), & x \notin B_\pi,
\end{cases}
\]
where $g_{\pi,\sigma} \in \TCzero_{d,s+1}$ and either $g_{\pi,0}$ or $g_{\pi,1}$ differs from $f^\star$ on at least $|B_\pi|/2$ points of $B_\pi$.
\end{enumerate}
\end{proposition}

\begin{proof}
Let the output gate of $C^\star$ be
\[
f^\star(x)=\mathbf 1\!\left[\sum_{j=1}^m w_j z_j(x) \ge \theta\right],
\]
where each $z_j$ is either an external input or the output of a subcircuit of depth at most $d-1$.
For the trigger, let
\[
t_1 := |\{i \in I : \pi_i = 1\}|,
\]
and define
\[
\mathrm{CHECK}_{I,\pi}(x)
:=
\mathbf 1\!\left[\sum_{i \in I : \pi_i = 1} x_i - \sum_{i \in I : \pi_i = 0} x_i \ge t_1\right].
\]
If $x \in B_\pi$, then each coordinate with $\pi_i=1$ contributes $1$ and each coordinate with $\pi_i=0$ contributes $0$, so the sum is exactly $t_1$ and the gate outputs $1$.
If $x \notin B_\pi$, then either some required-$1$ coordinate contributes $0$ or some required-$0$ coordinate contributes $-1$, so the sum is strictly less than $t_1$ and the gate outputs $0$.
This proves (i).

Now choose
\[
M := \sum_{j=1}^m |w_j| + |\theta| + 1.
\]
Because each $z_j(x) \in \{0,1\}$, the quantity $\sum_j w_j z_j(x)$ always lies in the interval
\[
\left[-\sum_j |w_j|,\ \sum_j |w_j|\right].
\]
Hence for every $x$,
\[
\sum_j w_j z_j(x) - M < \theta
\qquad\text{and}\qquad
\sum_j w_j z_j(x) + M \ge \theta.
\]

Define the modified output gate by
\[
\widetilde g_{\pi,0}(x)
:=
\mathbf 1\!\left[\sum_{j=1}^m w_j z_j(x) - M\,\mathrm{CHECK}_{I,\pi}(x) \ge \theta\right],
\]
and similarly
\[
\widetilde g_{\pi,1}(x)
:=
\mathbf 1\!\left[\sum_{j=1}^m w_j z_j(x) + M\,\mathrm{CHECK}_{I,\pi}(x) \ge \theta\right].
\]
Outside $B_\pi$ we have $\mathrm{CHECK}_{I,\pi}(x)=0$, so both modified output gates compute exactly $f^\star(x)$.
Inside $B_\pi$ we have $\mathrm{CHECK}_{I,\pi}(x)=1$; the first modified gate is forced to $0$ and the second is forced to $1$ by the choice of $M$.
Thus $\widetilde g_{\pi,\sigma}=g_{\pi,\sigma}$ for $\sigma \in \{0,1\}$.

The only new gate is $\mathrm{CHECK}_{I,\pi}$.
The output gate is replaced rather than duplicated, so the size increases by
at most $1$.  If the original circuit has depth $d^\star\le d$, then the
subcircuits feeding its output gate have depth at most $d^\star-1$.  Since
$\mathrm{CHECK}_{I,\pi}$ has depth $1$, the modified output gate has depth at
most
\[
1+\max\{1,d^\star-1\}\le d,
\]
using $d\ge 2$.  All gates remain threshold gates.
\end{proof}

\begin{remark}[Choosing one deceiver per block]
\label{rem:choose-one-per-block}
For each block $B_\pi$, choose $\sigma_\pi\in \{0,1\}$ so that
\[
|\{x \in B_\pi : f^\star(x) \neq \sigma_\pi\}| \ge |B_\pi|/2.
\]
Then $g_\pi := g_{\pi,\sigma_\pi}$ differs from $f^\star$ on at least one input in $B_\pi$ and agrees with $f^\star$ off $B_\pi$.
This allows us to use one deceiver per block in the elimination argument while keeping the construction explicit at the level of the existence proof.
\end{remark}

\begin{theorem}[Restatement of Theorem~\ref{thm:unbounded}]
\label{app:thm:unbounded}
Fix a constant depth $d\ge 2$.  For every input length $n$, every gate budget
$s$, every target $f^\star\in\TCzero_{d,s}$, and every trigger length
$1\le t\le n$,
\[
\cert\bigl(f^\star,\TCzero_{d,s+1}\bigr)\ge 2^t.
\]
Consequently, along any sequence of input lengths, if the chosen trigger
lengths satisfy $c n\le t(n)\le n$ for some fixed $c>0$, then the lower bound
is $2^{\Omega(n)}$. Moreover, if $s(n)=n^{O(1)}$, then for all sufficiently
large $n$ there exists $f_n\in\TCzero_{d,s(n)}$ such that
\[
\cert\bigl(f_n,\TCzero_{d,s(n)}\bigr)\le n^{O(1)}
\quad\text{while}\quad
\cert\bigl(f_n,\TCzero_{d,s(n)+1}\bigr)\ge 2^{\Omega(n)}.
\]
\end{theorem}

\begin{proof}
Choose a threshold circuit $C^\star$ of size at most $s$ and depth at most $d$
computing $f^\star$.  Fix any coordinate set $I\subseteq[n]$ with $|I|=t$.
For each $\pi\in\{0,1\}^I$, choose $\sigma_\pi$ as in
\Cref{rem:choose-one-per-block}, and let
$g_\pi:=g_{\pi,\sigma_\pi}$ be the corresponding one-gate override from
\Cref{prop:tc0-compat}.  Then $g_\pi\in\TCzero_{d,s+1}$ and by construction, $g_\pi$ agrees with $f^\star$ outside the trigger block $B_\pi$ and differs from $f^\star$ on at least one point of $B_\pi$.
Since the blocks $B_\pi$ are pairwise disjoint, \Cref{app:prop:block-elim} gives
\[
\cert\bigl(f^\star,\TCzero_{d,s+1}\bigr)\ge 2^t.
\]
The asymptotic consequence follows whenever $c n\le t(n)\le n$ for some fixed
$c>0$. It remains to justify the best-case, non-vacuity statement.  By
\Cref{app:prop:tc0-exists-poly-cert}, for each $n$ there exists some
$f_n\in\TCzero_{d,s(n)}$ satisfying
\[
\cert\bigl(f_n,\TCzero_{d,s(n)}\bigr)
\le O\bigl(s(n)(n+s(n))^2\bigr).
\]
When $s(n)=n^{O(1)}$, this bound is polynomial in $n$.  Applying the
target-wise lower bound above to this same $f_n$ with, say,
$t=\lfloor n/2\rfloor$, yields
\[
\cert\bigl(f_n,\TCzero_{d,s(n)+1}\bigr)
\ge 2^{\lfloor n/2\rfloor}
=2^{\Omega(n)}.
\]
Thus the exponential lower bound in the enlarged class is not inherited from a
target that was already hard to certify in the smaller class.
\end{proof}

\subsection{Certification hardness in $\ACzero$}
\label{app:ac0-overrides}

\begin{proposition}[Depth-preserving one-sided deceivers in $\ACzero$]
\label{prop:ac0-compat}

Let $C^\star$ be a depth-$d$ unbounded-fan-in Boolean circuit of size $s$
computing $f^\star$, under the free-literal $\ACzero$ convention of
\Cref{sec:prelim}, with $d \ge 2$.
Fix $I \subseteq [n]$ with $|I|=t$ and a pattern $\pi\in \{0,1\}^I$.

Define
\[
\mathrm{CHECK}_{I,\pi}(x)
:=
\bigwedge_{i \in I} \ell_i(x),
\qquad
\ell_i(x)=
\begin{cases}
x_i, & \pi_i=1,\\
\neg x_i, & \pi_i=0,
\end{cases}
\]
and
\[
\mathrm{MISS}_{I,\pi}(x)
:=
\bigvee_{i \in I} m_i(x),
\qquad
m_i(x)=
\begin{cases}
\neg x_i, & \pi_i=1,\\
x_i, & \pi_i=0.
\end{cases}
\]
Then:
\begin{enumerate}[label=(\roman*)]
\item $\mathrm{CHECK}_{I,\pi}(x)=1$ iff $x \in B_\pi$;
\item $\mathrm{MISS}_{I,\pi}(x)=0$ iff $x \in B_\pi$;
\item if the output gate of $C^\star$ is an OR gate, then
\[
g_{\pi,1}(x):= f^\star(x)\vee \mathrm{CHECK}_{I,\pi}(x)
\]
belongs to $\ACzero_{d,s+1}$, and
\[
g_{\pi,1}(x)=
\begin{cases}
1, & x \in B_\pi,\\
f^\star(x), & x \notin B_\pi;
\end{cases}
\]
\item if the output gate of $C^\star$ is an AND gate, then
\[
g_{\pi,0}(x):= f^\star(x)\wedge \mathrm{MISS}_{I,\pi}(x)
\]
belongs to $\ACzero_{d,s+1}$, and
\[
g_{\pi,0}(x)=
\begin{cases}
0, & x \in B_\pi,\\
f^\star(x), & x \notin B_\pi.
\end{cases}
\]
\end{enumerate}
\end{proposition}

\begin{proof}
The identities in (i) and (ii) are immediate from the definitions.

For (iii), if the top gate of $C^\star$ is OR, append $\mathrm{CHECK}_{I,\pi}$
as one more input to that same top OR gate. Outside $B_\pi$ we have
$\mathrm{CHECK}_{I,\pi}(x)=0$, so the circuit still computes $f^\star(x)$.
Inside $B_\pi$ we have $\mathrm{CHECK}_{I,\pi}(x)=1$, so the output is forced to $1$.
Only one new gate is added, namely $\mathrm{CHECK}_{I,\pi}$, and the depth remains $d$.

The proof of (iv) is dual: if the top gate of $C^\star$ is AND, append
$\mathrm{MISS}_{I,\pi}$ as one more input to that same top AND gate.
Outside $B_\pi$, $\mathrm{MISS}_{I,\pi}(x)=1$, so the output remains $f^\star(x)$.
Inside $B_\pi$, $\mathrm{MISS}_{I,\pi}(x)=0$, so the output is forced to $0$.
Again the size increases by at most $1$ and the depth remains $d$.
\end{proof}

One important aspect is that for $\ACzero$ circuits we use the standard De Morgan convention that
unbounded-fan-in AND/OR gates may take literals $x_i$ and $\neg x_i$ as
inputs, and input negations are not counted toward the gate budget.
Equivalently, NOT gates are allowed only at the input level and are free.

Under the alternative convention in which every negated literal must be computed by a counted NOT gate, the $\ACzero$ constant-slack statement should be weakened from $\ACzero_{d+1,s+2}$ to $\ACzero_{d+1,s+t+2}$ for a trigger set of size $t$.
\begin{theorem}[Constant-slack certification hardness in $\ACzero$]
\label{app:thm:ac0}
Work under the free-literal $\ACzero$ convention of \Cref{sec:prelim}, and fix a constant depth $d\ge 2$.  For every input length $n$, every gate budget $s$, every target $f^\star\in\ACzero_{d,s}$, and every trigger length
$1\le t\le n$,
\[
\cert\bigl(f^\star,\ACzero_{d+1,s+2}\bigr)\ge 2^t.
\]
Consequently, along any sequence of input lengths, if the chosen trigger lengths satisfy $c n\le t(n)\le n$ for some fixed $c>0$, then the lower bound is $2^{\Omega(n)}$.

Moreover, if $s(n)=n^{O(1)}$, then for all sufficiently large $n$ there exists $f_n\in\ACzero_{d,s(n)}$ such that
\[
\cert\bigl(f_n,\ACzero_{d,s(n)}\bigr)\le n^{O(1)}
\quad\text{while}\quad
\cert\bigl(f_n,\ACzero_{d+1,s(n)+2}\bigr)\ge 2^{\Omega(n)}.
\]
\end{theorem}

\begin{proof}
Choose an $\ACzero$ circuit $C^\star$ of size at most $s$ and depth at most $d$
computing $f^\star$.  Fix $I\subseteq[n]$ with $|I|=t$.
For each $\pi\in\{0,1\}^I$, choose $\sigma_\pi\in\{0,1\}$ so that
\[
|\{x\in B_\pi:f^\star(x)\ne \sigma_\pi\}|\ge |B_\pi|/2.
\]

Using the trigger gates from \Cref{prop:ac0-compat}, define the override as
follows.  If $\sigma_\pi=1$, use one trigger gate $\mathrm{CHECK}_{I,\pi}$ and one
new top OR gate:
\[
g_\pi(x)=f^\star(x)\vee \mathrm{CHECK}_{I,\pi}(x).
\]
If $\sigma_\pi=0$, use one trigger gate $\mathrm{MISS}_{I,\pi}$ and one new top
AND gate:
\[
g_\pi(x)=f^\star(x)\wedge \mathrm{MISS}_{I,\pi}(x).
\]
In both cases the trigger gate has unbounded fan-in and uses input literals
for free.  The construction adds one trigger gate and one combining gate above
the original circuit, so $g_\pi\in\ACzero_{d+1,s+2}$ and it agrees with $f^\star$ outside $B_\pi$ and is equal to the constant $\sigma_\pi$ on $B_\pi$, hence it differs from $f^\star$ on at least one point of $B_\pi$.
The trigger blocks are pairwise disjoint, so \Cref{app:prop:block-elim} gives
\[
\cert\bigl(f^\star,\ACzero_{d+1,s+2}\bigr)\ge 2^t.
\]
The asymptotic consequence follows whenever $c n\le t(n)\le n$ for some fixed
$c>0$. For the non-vacuity statement, apply \Cref{app:prop:poly-cert-ac0}: when
$s(n)=n^{O(1)}$, the smaller class $\ACzero_{d,s(n)}$ contains some
$f_n$ with
\[
\cert\bigl(f_n,\ACzero_{d,s(n)}\bigr)\le O\bigl(s(n)(n+s(n))\bigr)
=n^{O(1)}.
\]
Applying the target-wise lower bound above to this same $f_n$ with
$t=\lfloor n/2\rfloor$ gives
\[
\cert\bigl(f_n,\ACzero_{d+1,s(n)+2}\bigr)
\ge 2^{\lfloor n/2\rfloor}
=2^{\Omega(n)}.
\]
\end{proof}

\subsection{Certification hardness in $\NCone$}
\label{app:nc1-overrides}

\begin{proposition}[Bounded-fan-in constant-override deceivers]
\label{prop:nc1-compat}
Let $C^\star$ be a fan-in-$2$ Boolean circuit of size $s$ and depth $d$ computing $f^\star$.
Fix $I \subseteq [n]$ with $|I|=t\ge 1$ and a pattern $\pi\in \{0,1\}^I$.
For each $\sigma \in \{0,1\}$ there is a fan-in-$2$ Boolean circuit $C_{\pi,\sigma}$ of size at most $s+2t+2$ and depth at most $1+\max\{d,2+\lceil \log_2 t\rceil\}$ computing
\[
g_{\pi,\sigma}(x):=
\begin{cases}
\sigma, & x \in B_\pi,\\
f^\star(x), & x \notin B_\pi.
\end{cases}
\]
\end{proposition}

\begin{proof}
For each $i\in I$, define the literal
\[
\ell_i(x)=
\begin{cases}
x_i, & \pi_i=1,\\
\neg x_i, & \pi_i=0.
\end{cases}
\]
Build $\mathrm{CHECK}_{I,\pi}$ by feeding these literals into a balanced binary AND tree.
This uses at most $t$ NOT gates and $t-1$ AND gates, so the trigger circuit has at most $2t-1$ gates and depth at most $1+\lceil\log_2 t\rceil$.
It satisfies
\[
\mathrm{CHECK}_{I,\pi}(x)=1 \iff x\in B_\pi.
\]
For the two output overrides, we set
\[
g_{\pi,1}(x)=f^\star(x)\vee \mathrm{CHECK}_{I,\pi}(x) \quad \text{ and } \quad g_{\pi,0}(x)=f^\star(x)\wedge \neg \mathrm{CHECK}_{I,\pi}(x).
\]
For the $1$-override, outside $B_\pi$, the trigger is $0$ and the output is $f^\star(x)$. Inside $B_\pi$, the trigger is $1$ and the output is forced to $1$.
This adds one OR gate to the trigger circuit.
In contrast, for the $0$-override, outside $B_\pi$, the second input to the final AND is $1$ and the output is $f^\star(x)$. Inside $B_\pi$, it is $0$ and the output is forced to $0$.
This adds one NOT gate and one AND gate to the trigger circuit.
In both cases the total size is at most $s+2t+2$, and the stated depth bound follows from the balanced trigger tree and the final combining gate.
\end{proof}

\begin{theorem}[Bounded fan-in circuits with trigger-budget slack]
\label{app:thm:bounded-fanin}
Let $C^\star$ be a fan-in-$2$ Boolean circuit of size $s$ and depth $d$
computing $f^\star$.  Fix a trigger length $1\le t\le n$, and let
$\mathcal H^{\mathrm{BF}}_{n,s,d,t}$ be the class of functions computable by
fan-in-$2$ Boolean circuits of size at most $s+2t+2$ and depth at most
\[
1+\max\{d,2+\lceil \log_2 t\rceil\}.
\]
Then
\[
\cert\bigl(f^\star,\mathcal H^{\mathrm{BF}}_{n,s,d,t}\bigr)\ge 2^t.
\]
Consequently, along any sequence of input lengths, if the chosen trigger
lengths satisfy $c n\le t(n)\le n$ for some fixed $c>0$, then the lower bound
is $2^{\Omega(n)}$, at the cost of additive gate slack
$2t(n)+2=\Theta(n)$.
\end{theorem}

\begin{proof}
Fix $I\subseteq[n]$ with $|I|=t$.
For each $\pi\in\{0,1\}^I$, choose $\sigma_\pi\in\{0,1\}$ such that
\[
|\{x\in B_\pi:f^\star(x)\ne \sigma_\pi\}|\ge |B_\pi|/2.
\]
Let $g_\pi:=g_{\pi,\sigma_\pi}$ be the constant-override circuit from
\Cref{prop:nc1-compat}.  Then $g_\pi\in\mathcal H^{\mathrm{BF}}_{n,s,d,t}$ and $g_\pi$ agrees with $f^\star$ outside $B_\pi$ and differs from $f^\star$ on at least one point of $B_\pi$.
Since the trigger blocks $B_\pi$ are pairwise disjoint, \Cref{app:prop:block-elim} gives
\[
\cert\bigl(f^\star,\mathcal H^{\mathrm{BF}}_{n,s,d,t}\bigr)
\ge 2^t.
\]
\end{proof}

\begin{corollary}[$\NCone$ with additive slack depending on $n$]
\label{app:cor:nc1}
Let $s,k:\N\to\N$, with $k(n)\ge 4$ for all sufficiently large $n$.  For each
$n$, let
\[
\mathcal C^{\mathrm{NC}}_{n,s}:=\NCone_{s(n)},
\]
and let $\mathcal H^{\mathrm{NC}}_{n,s,k}$ be the corresponding enlarged
fan-in-$2$ class with size at most $s(n)+k(n)$ and depth $O(\log n)$.
For every target $f^\star\in\mathcal C^{\mathrm{NC}}_{n,s}$,
\[
\cert\!\bigl(f^\star,\mathcal H^{\mathrm{NC}}_{n,s,k}\bigr)
\ge
2^{\min\{n,\lfloor k(n)/4\rfloor\}}
=
2^{\Omega(\min\{n,k(n)\})}.
\]
In particular, if $k(n)=\Omega(n)$, then
\[
\cert\!\bigl(f^\star,\mathcal H^{\mathrm{NC}}_{n,s,k}\bigr)
\ge 2^{\Omega(n)}.
\]

Moreover, if $s(n)=n^{O(1)}$ and $k(n)=\Omega(n)$, then for all sufficiently
large $n$ there exists $f_n\in\mathcal C^{\mathrm{NC}}_{n,s}$ such that
\[
\cert\bigl(f_n,\mathcal C^{\mathrm{NC}}_{n,s}\bigr)\le n^{O(1)}
\quad\text{while}\quad
\cert\bigl(f_n,\mathcal H^{\mathrm{NC}}_{n,s,k}\bigr)\ge 2^{\Omega(n)}.
\]
\end{corollary}

\begin{proof}
Choose a fan-in-$2$ circuit $C^\star$ for $f^\star$ of size at most $s(n)$ and
depth $d(n)=O(\log n)$.  Set
\[
t:=\min\{n,\lfloor k(n)/4\rfloor\}.
\]
For all sufficiently large $n$, $k(n)\ge 4$, hence $1\le t\le n$.  By
\Cref{app:thm:bounded-fanin}, there are $2^t$ deceivers, one for each trigger
block $B_\pi$ with $\pi\in\{0,1\}^I$, and each deceiver is computed by a fan-in-$2$
circuit of size at most
\[
s(n)+2t+2\le s(n)+k(n).
\]
Each deceiver has depth at most
\[
1+\max\{d(n),2+\lceil\log_2 t\rceil\}=O(\log n),
\]
so the whole family lies in $\mathcal H^{\mathrm{NC}}_{n,s,k}$.  The
certificate lower bound follows from \Cref{app:prop:block-elim}.

For the final statement, apply \Cref{app:prop:poly-cert-nc1} to the smaller
class.  When $s(n)=n^{O(1)}$, it contains some $f_n$ with polynomial
certificate size inside $\mathcal C^{\mathrm{NC}}_{n,s}$.  Applying the
target-wise lower bound above to this same $f_n$ and using $k(n)=\Omega(n)$
gives the exponential lower bound in the enlarged class.
\end{proof}

%% file: 992_transformer.tex
\section{Additional theory for Transformers}
\label{app:transformers}

\subsection{Conventions and definitions}
\label{app:conv-def}

\textbf{Numerical precision.}
Following \citet{merrill2023parallelism,merrill2023logic},
we represent parameters and activations using a finite
$p$-precision datatype $D_p\subset\mathbb R$, with
$|D_p|\le 2^{O(p)}$.
We write $[F]_p$ for the result of evaluating an expression $F$
at precision $p$, with rounding and saturation on overflow. The datatype has a maximum representable magnitude $F_p^{\max}$,
with overflow saturating at $\pm F_p^{\max}$.
Ordinary arithmetic uses a fixed evaluation order, with rounding
after each scalar operation.
Following the framework of $p$-precise operations
\citep{merrill2025padding}, we treat layer normalization and
the hard attention average as single operations, whose
results are coordinatewise rounded at the end.
The deceiver construction in the Transformer certification theorem preserves the original precision and uses two successive additions in the final readout, as detailed in \Cref{app:thm:transformer-deceiver}.

\begin{definition}[Projected pre-norm \citep{merrill2024expressive}]
\label{def:projected-prenorm}
For a vector $u\in D_p^r$, $r\ge2$, with nonzero centered norm,
define
\[
\mathrm{LN}_p(u)
:=
\left[
\frac{u-\bar u\mathbf 1}
{\|u-\bar u\mathbf 1\|_2}
\right]_p,
\qquad
\bar u:=\frac1r\sum_{a=1}^r u_a.
\]
Given a residual state $h\in D_p^m$ and a projection
$P\in D_p^{r\times m}$, the projected pre-norm operation is
\[
\mathrm{PN}_{p,P}(h)
:=\mathrm{LN}_p([Ph]_p).
\]
\end{definition}

\begin{definition}[Projected-pre-norm finite-precision AHAT class]
\label{def:general-ahat-class}
Fix an input length $n\ge1$, a precision $p$, and an architecture
$\mathsf A_n$ specifying $L\ge1$ layers, residual width $m\ge1$,
the number and dimensions of attention heads, projection
dimensions, feedforward widths, and arithmetic evaluation graphs.
All projection dimensions are at least two.
The embeddings, weight matrices, and biases are variable
elements of $D_p$, while the architecture and evaluation
graphs remain fixed.

\textbf{Initialization and attention.}
A token--position embedding table
$e:\{0,1\}\times[n]\to D_p^m$ initializes the hidden states
\[
h_i^0=e(x_i,i),\qquad i\in[n].
\]
For each head $k$ in layer $\ell$, define
\begin{align*}
u_i^{\ell,k}
&=\mathrm{PN}_{p,P_{\ell,k}}(h_i^{\ell-1}),
&\quad
q_i^{\ell,k}
&=[W^Q_{\ell,k}u_i^{\ell,k}+b^Q_{\ell,k}]_p,\\
k_i^{\ell,k}
&=[W^K_{\ell,k}u_i^{\ell,k}+b^K_{\ell,k}]_p,
&\quad
v_i^{\ell,k}
&=[W^V_{\ell,k}u_i^{\ell,k}+b^V_{\ell,k}]_p.
\end{align*}
Queries and keys have the same dimension, and all three maps
use the same projected normalized representation.
Attention is unmasked and uses unscaled dot-product scores:
\[
s_{ij}^{\ell,k}
=[\langle q_i^{\ell,k},k_j^{\ell,k}\rangle]_p,
\qquad
J_i^{\ell,k}
=\operatorname*{arg\,max}_{j\in[n]}s_{ij}^{\ell,k}.
\]
The head output is
\[
o_i^{\ell,k}
=
\left[
\frac{1}{|J_i^{\ell,k}|}
\sum_{j\in J_i^{\ell,k}}v_j^{\ell,k}
\right]_p.
\]
All maximizing positions are included, and their values
are averaged before rounding.

\textbf{Residual and feedforward updates.}
Concatenate the head outputs into $o_i^\ell$ with attention and feedforward updates being
\begin{align*}
a_i^\ell
&=[h_i^{\ell-1}+W^O_\ell o_i^\ell+b^O_\ell]_p,\\
r_i^\ell
&=\mathrm{PN}_{p,P^{\rm ff}_\ell}(a_i^\ell),\\
h_i^\ell
&=[a_i^\ell+W_{2,\ell}\operatorname{ReLU}
([W_{1,\ell}r_i^\ell+b_{1,\ell}]_p)+b_{2,\ell}]_p.
\end{align*}

\textbf{Readout and hypothesis class.}
The representation at the final position is mapped to
a scalar logit, whose sign determines the binary output:
\[
\lambda_T^{(p)}(x)
=[\langle a,h_n^L\rangle+b]_p,
\qquad
f_T(x)=\mathbf 1[\lambda_T^{(p)}(x)\ge0].
\]
The readout evaluation order is fixed by the architecture.

A parameter assignment is admissible if every projected pre-norm operation, including those in the feedforward blocks, is defined for every input $x\in\{0,1\}^n$.
The class $\mathcal T_{\mathsf A_n,p}(n)$ consists of all Boolean functions computed by admissible $D_p$-valued parameter assignments.
\end{definition}

\subsection{Polynomial-size certificates exist in log-precision Transformers}

\begin{proposition}[Polynomial certificate size for fixed architectures]
\label{app:prop:poly-cert-transformer}
Fix an input length $n$ and an architecture $\mathsf A_n$
as in \Cref{def:general-ahat-class}.
Let $P_{\mathsf A}(n)$ be its number of variable $D_p$-valued
scalar parameters, including the $2nm$ entries of the
embedding table.
Assume $\mathcal T_{\mathsf A_n,p}(n)$ is nonempty.
Then some $f\in\mathcal T_{\mathsf A_n,p}(n)$ satisfies
\[
\cert\bigl(f,\mathcal T_{\mathsf A_n,p}(n)\bigr)
\le O(P_{\mathsf A}(n)p).
\]
In particular, polynomially many parameter slots and
$p(n)=O(\log n)$ give a target with a polynomial-size certificate.
\end{proposition}

\begin{proof}
Each admissible parameter assignment determines one Boolean
function. Counting all assignments, including inadmissible ones,
gives
\[
|\mathcal T_{\mathsf A_n,p}(n)|
\le |D_p|^{P_{\mathsf A}(n)}
\le 2^{O(P_{\mathsf A}(n)p)}.
\]
The claim follows from \Cref{app:lem:small-cert-exists}.
\end{proof}

\subsection{Transformer block-override construction}

We now construct the alternatives used in the certification
lower bound. Projected pre-norm lets the original heads and
feedforward blocks ignore the added residual coordinates,
while a new head reads a separate four-dimensional code.
The argument is at a fixed input length and uses the embedding
table allowed by \Cref{def:general-ahat-class}.

\begin{theorem}[Transformer block-override construction]
\label{app:thm:transformer-deceiver}
Fix $n\ge1$ and let $T$ be a Transformer of architecture
$\mathsf A_n$ as in \Cref{def:general-ahat-class}, with $d\ge1$
layers, residual width $m$, and precision $p\ge p_0$.
There exists an extension $\mathsf A_n^+$ of
$\mathsf A_n$ that adds one attention head to the final layer
and six residual coordinates, while preserving the original
depth, precision, and inherited arithmetic computations.

Fix $I\subseteq[n]$, $\pi\in\{0,1\}^I$, and define
\[
B_{I,\pi}:=\{x\in\{0,1\}^n:x_I=\pi\}.
\]
For every $\sigma\in\{0,1\}$, there exists a Transformer
$T_{I,\pi,\sigma}$ of architecture $\mathsf A_n^+$ such that
\[
f_{T_{I,\pi,\sigma}}(x)=
\begin{cases}
\sigma,&x\in B_{I,\pi},\\[1mm]
f_T(x),&x\notin B_{I,\pi}.
\end{cases}
\]
Moreover, the original computation is preserved exactly
on the first $m$ hidden-state coordinates at every layer.
\end{theorem}

\begin{proof}
We first give the construction for $t:=|I|\ge1$.

\paragraph{Step 1: an exact four-dimensional code.}
Define
\[
\alpha:=\Bigl(\frac12,\frac12,-\frac12,-\frac12\Bigr),\qquad
\beta:=\Bigl(\frac12,-\frac12,\frac12,-\frac12\Bigr),\qquad
\gamma:=\Bigl(\frac12,-\frac12,-\frac12,\frac12\Bigr).
\]
These vectors are exactly representable, have mean zero and
Euclidean norm one, and are therefore fixed by $\mathrm{LN}_p$.
For
\[
q^\dagger:=(2,-1,1,0),\qquad v(c):=-c_2-c_3,
\]
we have
\[
\begin{aligned}
q^\dagger\cdot\alpha&=0,&
q^\dagger\cdot\beta&=2,&
q^\dagger\cdot\gamma&=1,\\
v(\alpha)&=0,&v(\beta)&=0,&v(\gamma)&=1.
\end{aligned}
\]
The scalar calculations above are exact at precision $p\ge p_0$.
We use $\alpha$ for unselected positions, $\beta$ for mismatches,
and $\gamma$ for matches.

\paragraph{Step 2: augmented hidden state and embedding.}
Extend each residual state from $m$ to $m+6$ coordinates,
written as
\[
h_i=(u_i,C_i,z_i^{(1)},z_i^{(2)})
\in D_p^m\times D_p^4\times D_p\times D_p.
\]
Here $u_i$ holds the original state, $C_i$ holds the code,
and the two final coordinates will store copies of the trigger bit.
For each position $i\in[n]$ and input bit $a\in\{0,1\}$, set
\[
C_i(a):=
\begin{cases}
\alpha,&i\notin I,\\[1mm]
\beta,&i\in I\text{ and }a\ne\pi_i,\\[1mm]
\gamma,&i\in I\text{ and }a=\pi_i,
\end{cases}
\qquad
 e'(a,i):=(e(a,i),C_i(a),0,0).
\]
This is a token--position table of the form allowed by the
extended architecture.

\paragraph{Step 3: preserve the inherited computation.}
For every original head, extend its projection by zero columns:
\[
P_{\ell,k}^+
=\bigl[P_{\ell,k}\mid0_{r_{\ell,k}\times6}\bigr],
\]
where $r_{\ell,k}$ is the projected-view dimension.
Keep its query, key, and value maps unchanged.
Likewise, extend each feedforward projection as
\[
(P_\ell^{\rm ff})^+
=\bigl[P_\ell^{\rm ff}\mid0_{r_\ell^{\rm ff}\times6}\bigr].
\]
Pad the attention output and feedforward output matrices and
biases with zeros in the added coordinates.
The original normalized views and updates are therefore unchanged,
and the added coordinates receive zero updates from these blocks.

The new head is present only in the final layer.
Thus, for $0\le\ell\le d-1$,
\begin{equation}
\label{eq:ppn-invariant-new}
h_i^\ell(x)
=\bigl(\widetilde h_i^\ell(x),C_i(x_i),0,0\bigr),
\end{equation}
where $\widetilde h_i^\ell(x)$ is the state of the original
Transformer $T$.
The same extension is used for the inherited heads and
feedforward block in layer $d$.

\paragraph{Step 4: detect the block pattern}
The new final-layer head $G$ selects the code coordinates:
\[
P_G=\bigl[0_{4\times m}\mid I_4\mid0_{4\times2}\bigr].
\]
By \eqref{eq:ppn-invariant-new},
$\mathrm{PN}_{p,P_G}(h_j^{d-1})=C_j(x_j)$.
On this shared normalized view, use
\[
q_G(c)=q^\dagger,\qquad k_G(c)=c,\qquad v_G(c)=-c_2-c_3.
\]
The constant query is an affine map with zero weight matrix
and bias $q^\dagger$.
Its scores are independent of the query position:
\[
s^G_{ij}=[\langle q^\dagger,C_j(x_j)\rangle]_p
=
\begin{cases}
0,&j\notin I,\\
2,&j\in I\text{ and }x_j\ne\pi_j,\\
1,&j\in I\text{ and }x_j=\pi_j.
\end{cases}
\]
If any selected position mismatches, precisely those mismatches
maximize the score, and all have value zero.
Otherwise, the maximizers are the selected positions,
which all have value one.
Since $I$ is nonempty and AHAT preserves a mean of identical
representable values, the head returns
\[
z'_i=\mathbf 1[x\in B_{I,\pi}]
\qquad\text{for every }i\in[n].
\]

The final-layer output projection writes $z'_i$ into both
$z_i^{(1)}$ and $z_i^{(2)}$, and writes zero from the added
head into the other coordinates.
The inherited heads still update only the first $m$ coordinates.
The feedforward block also reads and updates only the original
coordinates, so the final state is
\[
h_i^d(x)=
\bigl(\widetilde h_i^d(x),C_i(x_i),z'_i,z'_i\bigr).
\]
All inherited normalizations remain defined, and the added
head normalizes one of the nonconstant vectors
$\alpha,\beta,\gamma$. The enlarged parameter assignment is
therefore admissible.

\paragraph{Step 5: Override the output.}
Let $\lambda_T^{(p)}(x)$ be the original logit. 
Set $F:=F_p^{\max}$, so that $|\lambda_T^{(p)}(x)|\le F$.
The new readout computes the original logit first and then
adds two override terms successively:
\[
L_\sigma(x):=
\left[
\left[
\lambda_T^{(p)}(x)+(2\sigma-1)Fz_n^{(1)}
\right]_p
+(2\sigma-1)Fz_n^{(2)}
\right]_p.
\]
Its output is $\mathbf 1[L_\sigma(x)\ge0]$.
The additional readout weights are $(2\sigma-1)F\in D_p$.

Outside $B_{I,\pi}$, both trigger coordinates vanish, so
$L_\sigma(x)=\lambda_T^{(p)}(x)$ and the original label is preserved.
On $B_{I,\pi}$, both equal one.
For $\sigma=1$, the first addition gives a nonnegative value;
adding $F$ again gives $F$ after saturation.
For $\sigma=0$, the first subtraction gives a nonpositive value;
subtracting $F$ again gives $-F$ after saturation.
Thus the readout returns $\sigma$ throughout the block.

If $I=\varnothing$, use $C_i(a)=\alpha$ at every
position and a constant final logit equal to $2\sigma-1$.
The same hidden-state extension remains admissible and
preserves the original hidden computation.
In all cases, the construction adds one final-layer head and
six residual coordinates, with no increase in depth or precision.
\end{proof}

\begin{corollary}[Transformer certificate size]
\label{app:cor:transformer}
Fix $n\ge1$, $p\ge p_0$, and
$f^\star\in\mathcal T_{\mathsf A_n,p}(n)$.
Let $\mathsf A_n^+$ be the extension in
\Cref{app:thm:transformer-deceiver}.
For every $1\le t\le n$,
\[
\cert\bigl(f^\star,\mathcal T_{\mathsf A_n^+,p}(n)\bigr)
\ge 2^t.
\]
Taking $t=\Theta(n)$ gives an exponential lower bound.
\end{corollary}

\begin{proof}
Choose a Transformer $T^\star$ of architecture $\mathsf A_n$
computing $f^\star$.
The enlarged architecture also computes $f^\star$ by setting
both override weights to zero in the nonempty-block construction.

Fix $I\subseteq[n]$ with $|I|=t$.
For each $\pi\in\{0,1\}^I$, choose a label $\sigma_\pi$
that disagrees with $f^\star$ on at least half of $B_{I,\pi}$.
By \Cref{app:thm:transformer-deceiver}, the enlarged class
contains a function $g_\pi$ equal to $f^\star$ outside
$B_{I,\pi}$ and to $\sigma_\pi$ on that block.
Consequently, its disagreement set
\[
E_\pi:=\{x:g_\pi(x)\ne f^\star(x)\}
\]
is nonempty and contained in $B_{I,\pi}$.
These blocks, and hence the disagreement sets, are pairwise
disjoint. There are $2^t$ of them, so
\Cref{app:prop:block-elim} gives the claim.
\end{proof}

%% file: 993_experiments.tex
\section{Experiment details}
\label{app:experiments}

This appendix records the details needed to reproduce and interpret the experiments in \Cref{sec:experiments}. The circuit experiment is a deterministic instantiation of the trigger construction. The Transformer experiment uses training only to generate strong candidate recognizers, which are then audited by finite held-out tests.

\subsection{Constructed threshold-circuit experiment}
\label{app:circuit-experiments}

The circuit experiment uses the full domain $\mathcal X=\{0,1\}^{3n+1},$
where an input is a triple $(a,b,z)$ with $a,b\in\{0,1\}^n$ and $z\in\{0,1\}^{n+1}$. The target accepts exactly when $z=a+b$. 

Let $A,B,Z$ denote the integers encoded by $a,b,z$, and set
$\Delta=Z-A-B$ and $M=2^{n+1}$. The target is computed by two
threshold gates:
\[
\begin{aligned}
q(a,b,z)&=\mathbf 1[\Delta\ge1],\\
f^\star(a,b,z)&=\mathbf 1[\Delta-Mq(a,b,z)\ge0].
\end{aligned}
\]
The output gate reads both $q$ and the original input bits.
Since $-(M-2)\le\Delta\le M-1$, its score is zero when
$\Delta=0$ and negative otherwise. Two gates are necessary for $n\ge1$.
Fixing all input bits except the least significant bits $b_0,z_0$ to zero
restricts the target to equality of two bits, which no single threshold
gate computes.

For a trigger pattern $\pi\in\{0,1\}^n$, the test
$t_\pi(a)=\mathbf 1[a=\pi]$ is a single threshold gate. Define
\[
g_\pi(a,b,z)=
\mathbf 1[\Delta-Mq(a,b,z)+Mt_\pi(a)\ge0].
\]
Because $1-M\le\Delta-Mq(a,b,z)\le0$, this circuit accepts every input
with $a=\pi$ and agrees with the target outside this block.
Each deceiver therefore uses three threshold gates at depth $2$. The block
\[
B_\pi=\{(a,b,z):a=\pi\}
\]
has size $2^{2n+1}$. Inside this block, the target accepts exactly one value of $z$ for each fixed $b$, while $g_\pi$ accepts all $2^{n+1}$ candidates. Hence each constructed deceiver has
\[
e_n=|E_\pi|=2^n(2^{n+1}-1)
\]
errors, all of them false positives. The error sets are disjoint across different $\pi$. For a uniformly sampled candidate certificate $S_m\subseteq\mathcal X$, sampled without replacement, the pass probability of one constructed deceiver is therefore
\[
p_n(m)=
\binom{|\mathcal X|-e_n}{m}\Big/\binom{|\mathcal X|}{m}.
\]
Since there are $2^n$ trigger patterns, the expected number of surviving constructed deceivers is
\[
2^n p_n(m).
\]
This is the quantity plotted in the left panel of \Cref{fig:addition-side-by-side}. The calculation is exact for this finite construction. It uses uniform sampling of certificate points, but it does not assume that errors are uniform.

The disjoint error sets also give a bound independent of the sampling rule. Each example labeled by the target eliminates at most one constructed deceiver. Thus any set of $m$ labeled examples leaves at least $\max\{2^n-m,0\}$ of them consistent, and eliminating all of them requires at least $2^n$ examples. For $m=\poly(n)$, the surviving fraction is at least $1-m/2^n$, which tends to one as $n$ grows.

\subsection{Transformer experiment details}
\label{app:transformers-addition}

The Transformer experiment asks whether trained recognizers can leave small residual error sets after strong intermediate supervision and error-free training and validation checks.
For consistency with the main text, we call accepted models with nonzero held-out test error \emph{deceivers}. This name does not imply that they behave like the constructed trigger deceivers.

\subsubsection{Task and data}
\label{app:addition-task}

The task is binary-addition recognition. For $n\in\{8,10,12\}$,
operands $a,b\in\{0,1\}^n$, and a candidate sum
$z\in\{0,1\}^{n+1}$, the label is
\[
y=\mathbf 1[z=a+b].
\]
We split operand pairs $(a,b)$ into training, validation, and
held-out test sets, reserving $25\%$ of the pairs for testing.
The same split is used for all models at each bit length.

For each operand pair, we generate one positive example and all
$n+1$ one-bit flips of the correct sum, yielding $n+2$ labeled
examples. This avoids using random negative outputs that may be
far from the correct sum. \Cref{tab:addition-split-sizes}
reports the resulting dataset sizes.

\begin{table*}[t]
\centering
\small
\caption{Full operand-pair splits and corresponding labeled-example pools. Each pair contributes $n+2$ examples. Training samples minibatches from its pool.}
\label{tab:addition-split-sizes}
\begin{tabular}{crrrrrr}
\toprule
& \multicolumn{3}{c}{Operand pairs}
& \multicolumn{3}{c}{Labeled examples} \\
\cmidrule(lr){2-4}
\cmidrule(lr){5-7}
$n$ & Training & Validation & Test
    & Training & Validation & Test \\
\midrule
8
& $45{,}875$ & $3{,}277$ & $16{,}384$
& $458{,}750$ & $32{,}770$ & $163{,}840$ \\
10
& $734{,}003$ & $52{,}429$ & $262{,}144$
& $8{,}808{,}036$ & $629{,}148$ & $3{,}145{,}728$ \\
12
& $11{,}744{,}051$ & $838{,}861$ & $4{,}194{,}304$
& $164{,}416{,}714$ & $11{,}744{,}054$ & $58{,}720{,}256$ \\
\bottomrule
\end{tabular}
\end{table*}

The held-out set contains $2^{2n-2}$ operand pairs and
$N_{\rm test}=(n+2)2^{2n-2}$ labeled examples.

\subsubsection{Architecture and input format}
\label{app:addition-architecture}

All models use an encoder-only Transformer with four
pre-layer-normalized layers, hidden size $d_{\mathrm{model}}=64$, two attention heads, feed-forward dimension $256$, and no dropout. The models have $204{,}491$, $205{,}131$, and $205{,}771$ trainable parameters at $n=8,10,12$, respectively. The small differences arise from length-dependent embeddings and column-query tokens.

The input sequence is
\[
[\mathrm{CLS}],\,a_{n-1},\ldots,a_0,\,[\mathrm{SEP}_B],\,
b_{n-1},\ldots,b_0,\,[\mathrm{SEP}_Z],\,z_n,\ldots,z_0,
\]
with bits presented in most-significant-bit order.

\subsubsection{Intermediate supervision}
\label{app:addition-training}

We use intermediate supervision to obtain strong candidate models. The supervision supplies arithmetic targets during training. These targets are not provided at evaluation.

The model processes a $z$-blind stream: the candidate output bits are masked, and $n+1$ learned column-query tokens are appended. The encoder is trained to predict the true sum bits together with the column states used to compute them. Although the input tokens are presented in most-significant-bit order, the arithmetic states below are indexed from least to most significant bit. We set $a^{(n)}=b^{(n)}=0$ and $c_0=0$. For output column
$k\in\{0,\ldots,n\}$,
\[
s_k=a^{(k)}+b^{(k)},\qquad
u_k=s_k+c_k,\qquad
z^\star_k=u_k\bmod 2,\qquad
c_{k+1}=\lfloor u_k/2\rfloor.
\]
For $k=n$, this emits the final carry bit.

The training objective is
\[
\mathcal L
=
\mathcal L_{\mathrm{accept}}
+0.25\,\mathcal L_s
+\mathcal L_u
+0.5\,\mathcal L_c
+0.25\,\mathcal L_z.
\]
Here $\mathcal L_{\mathrm{accept}}$ is binary cross-entropy averaged
equally over positive and negative examples.
The terms $\mathcal L_s$, $\mathcal L_u$, and $\mathcal L_c$ are
mean cross-entropy losses for $s_k$, $u_k$, and the outgoing carry
$c_{k+1}$, respectively. The term $\mathcal L_z$ is mean binary
cross-entropy for the true sum bits $z^\star_k$.

At evaluation, the model receives only $(a,b,z)$.
No arithmetic targets are supplied.
It computes sum-bit logits from the $z$-blind stream, compares
them with the candidate bits, and combines the resulting mismatch
logits using a log-sum-exp aggregation with temperature $1$.
This produces a scalar accept/reject logit.

We train with AdamW using learning rate $3\times10^{-4}$,
weight decay $0.01$, and gradient-norm clipping at $1$.
The learning rate increases linearly during the first $500$
updates, then follows a cosine schedule reaching
$3\times10^{-5}$ at update $80{,}000$.
The batch size is $512$, rounded down
to complete operand-pair groups. %

\subsubsection{Acceptance process}
\label{app:addition-acceptance}

Before training begins, we fix training and validation checks
from their respective splits. The training check contains
$20{,}000$ operand pairs at every bit length. The validation check contains $3{,}277$ pairs at $n=8$ and $20{,}000$ pairs at
$n=10,12$. Each pair contributes $n+2$ labeled examples: one positive and $n+1$ one-bit negatives.

Every $500$ training steps, we evaluate positive and negative
accuracy separately on both fixed checks. A checkpoint is accepted only when both checks contain zero errors. We then evaluate the accepted checkpoint on the held-out test set.

\subsubsection{Held-out test audit}
\label{app:addition-certification}

An accepted model is called \emph{test-exact} if it makes zero
mistakes on the full held-out test set of $N_{\rm test}$ examples.
This is exactness relative to the held-out set, not over all possible
triples $(a,b,z)$. The full domain contains $2^{3n+1}$ triples,
whereas the test set contains only the correct output and one-bit
output flips for held-out operand pairs.

At each of $n=8,10,12$, we evaluate $150$ accepted models.
Of these, $1$, $10$, and $44$ have held-out errors, respectively.
The single $n=8$ deceiver makes $2$ test errors.
At $n=10$, deceivers make between $1$ and $29$ test errors,
with a median of $5$. At $n=12$, they make between $1$ and
$292$ test errors, with a median of $15.5$.

At $n=10$, all ten deceivers have pairwise-disjoint error sets
on the held-out set. At $n=12$, a subfamily of $40$ of the $44$
deceivers has pairwise-disjoint error sets.
Together with certificates of these sizes, these disjoint subfamilies establish that the minimum numbers of held-out examples needed to eliminate all observed deceivers are exactly $10$ and $40$ at $n=10$ and $n=12$, respectively.

\subsubsection{Survivor calculation for trained models}
\label{app:addition-survivors}

For a fixed bit length $n$, let $\mathcal D$ be the set of
Transformer deceivers. Its size is $1$, $10$, or $44$ for
$n=8,10,12$, respectively. For a finite sampling set $Q$, define
\[
E_h^Q=\{x\in Q:h(x)\ne f^\star(x)\}.
\]
For $0\le m\le|Q|$, let
$S_m$ be a uniform size-$m$ subset of $Q$. Here sampling is over
individual labeled examples, not operand-pair groups.
Define the surviving deceiver count
\[
N_Q(S_m)
=
\sum_{h\in\mathcal D}
\mathbf 1[S_m\cap E_h^Q=\varnothing].
\]
A model survives exactly when the sample contains none of its
errors. Its survival probability is therefore
\[
p_h^Q(m)
:=
\Pr[S_m\cap E_h^Q=\varnothing]
=
\frac{\binom{|Q|-|E_h^Q|}{m}}{\binom{|Q|}{m}}.
\]
A numerator binomial coefficient is zero when $m>|Q|-|E_h^Q|$.
By linearity of expectation, the expected surviving count is
\[
\mathbb E[N_Q(S_m)]
=
\sum_{h\in\mathcal D}p_h^Q(m).
\]

The expectation is over the sampled set alone and follows by
linearity, without assuming independent survival events.
This is the quantity plotted in the right panel of
\Cref{fig:addition-side-by-side}. It depends on the individual
error counts, not on overlap between models' error sets.

The full-test curve uses $Q=\mathcal X_{\rm test}$.
To examine whether a simple structural restriction improves
sampling, we also use a compact subset $Q_{\rm tgt}$ containing
all positives and one-bit output flips in positions
$n-2,n-1,n$, indexed from the least significant bit.
This gives four examples per held-out operand pair.
At $n=10$, its size is
\[
|Q_{\rm tgt}|=262{,}144\cdot4=1{,}048{,}576,
\]
one third of the full test set. The compact subset contains $2/2$, $65/72$, and
$2{,}011/2{,}515$ of the distinct recorded error locations
at $n=8,10,12$, respectively. These correspond to
$100\%$, $90.3\%$, and $80.0\%$. These percentages count the union of error locations across
deceivers, showing that the errors concentrate in this structural subset, which occupies only $4/(n+2)$ of the full test pool.

Both sampling rules use the same deceiver family $\mathcal D$.
The targeted calculation replaces $Q$ by $Q_{\rm tgt}$ and
restricts each error set accordingly.
Both curves use the horizontal coordinate $m/N_{\rm test}$,
including for targeted sampling. The targeted curves therefore
end at
\[
m/N_{\rm test}=4/(n+2),
\]
when their sampling pools are exhausted.
In \Cref{fig:addition-side-by-side}, both axes are logarithmic, and zero expected counts are omitted.
The plot shows that the compact sampler reduces the expected surviving count in the displayed comparison.
This does not imply that it detects every model.
If $E_h^{Q_{\rm tgt}}=\varnothing$, then
$p_h^{Q_{\rm tgt}}(m)=1$ for every feasible sample size.
One $n=12$ deceiver has this property, accounting for the
single survivor at the endpoint of its targeted curve.

Even detecting all models that have errors in the compact
subset can require extensive sampling under this rule.
At $n=10$ and $n=12$, some models have exactly one error
in $Q_{\rm tgt}$. For each such model,
the detection probability is $m/|Q_{\rm tgt}|$. Thus, detecting
all models with errors in this subset with probability at least
$95\%$ requires sampling at least $95\%$ of it.
This is a limitation of uniform sampling,
not a lower bound on the minimum certificate size.

Smaller structurally defined subsets might perform better.
At the post-hoc extreme, the union
\[
U=\bigcup_{h\in\mathcal D}E_h^{\mathcal X_{\rm test}}
\]
eliminates every observed deceiver when queried in full.
Its sizes are $2$, $72$, and $2{,}515$ at $n=8,10,12$,
respectively. Smaller certificates suffice.
Choosing one error location per deceiver gives a set
$S_{\rm cert}$ with $|S_{\rm cert}|\le|\mathcal D|$.

Let $\mathcal H_{\rm obs}$ consist of the target and the observed
candidates restricted to $\mathcal X_{\rm test}$.
The minimum certificate sizes are
\[
\cert\!\left(
f^\star|_{\mathcal X_{\rm test}},\mathcal H_{\rm obs}
\right)
=
\begin{cases}
1, & n=8,\\
10, & n=10,\\
40, & n=12.
\end{cases}
\]
Pairwise-disjoint subfamilies give the lower bounds, while the explicit certificates achieve them.
In particular, a $40$-example certificate eliminates all
$44$ observed deceivers at $n=12$.

These are exact results for the observed families with queries
restricted to the held-out sets. Disjoint errors there do not
rule out shared errors elsewhere in the full domain.
The certificates use post-hoc knowledge of error locations,
whereas the plotted curves describe uniform sampling without
targeting individual known mistakes.
Neither these fixed-$n$ curves nor the compact-pool sizes
establish an asymptotic lower bound for the minimum certificate
of the trained-model class.

\subsubsection{Detecting all deceivers}
\label{app:addition-all-models}

The survivor curves report an expected count. To detect every
deceiver with probability at least $95\%$, we instead need the
probability that a single sample contains an error from each model.
Write $N=|\mathcal X_{\rm test}|$ and
$E_h=E_h^{\mathcal X_{\rm test}}$. For a uniform size-$m$ sample
without replacement from the full held-out set, this probability is
\[
P_{\rm all}(m)
=
\sum_{J\subseteq\mathcal D}(-1)^{|J|}
\frac{\binom{N-\left|\bigcup_{h\in J}E_h\right|}{m}}
     {\binom{N}{m}}.
\]
This is inclusion--exclusion: for each subfamily $J$, the ratio is
the probability of missing all its errors. The empty-subfamily
term is $1$, and an infeasible numerator binomial coefficient is
zero. No independence assumption is needed.

At $n=10$, the test set has $N=3{,}145{,}728$ examples and the ten
deceivers have pairwise-disjoint error sets of sizes
\[
(1,2,2,2,2,8,8,9,9,29).
\]
Each union size in the formula is therefore the sum of the
corresponding error counts. The model with only one error already
requires sampling at least $95\%$ of the test set for a $95\%$
chance of detection, since its error is included with probability
$m/N$.

Evaluating the formula for all ten deceivers gives
\[
m_{0.95}
:=\min\{m:P_{\rm all}(m)\ge0.95\}
=3{,}010{,}601,
\]
or approximately $95.7\%$ of the full test set. In contrast, ten appropriately selected labeled examples suffice to eliminate all ten deceivers.